\documentclass[letterpaper,journal]{IEEEtran}
\usepackage{microtype}
\usepackage{hyperref}
\hypersetup{hidelinks}
\usepackage{amsmath,amsfonts}
\usepackage{amsthm}
\usepackage{multirow}
\usepackage{adjustbox}
\usepackage{booktabs}
\usepackage{makecell}
\usepackage{algorithmic}
\usepackage{booktabs}
\usepackage{array}
\usepackage[table]{xcolor}
\usepackage[caption=false,font=normalsize]{subfig}
\usepackage{textcomp}
\usepackage[utf8]{inputenc}
\usepackage{amssymb}
\usepackage{stfloats}
\usepackage{url}
\usepackage{verbatim}
\usepackage{graphicx}
\usepackage{cite}
\usepackage{stackengine}
\usepackage{amsmath}
\usepackage{bbm}
\usepackage{mathtools}
\usepackage[ruled,vlined]{algorithm2e}
\usepackage{booktabs}
\usepackage{xcolor}
\usepackage{titlesec}
\titleformat{\subparagraph}[runin]{\normalfont\normalsize\bfseries}{\thesubparagraph}{1em}{}[.]
\titlespacing*{\section}{0pt}{10pt plus 2pt minus 2pt}{6pt plus 1pt minus 1pt}
\titlespacing*{\subsection}{0pt}{8pt plus 2pt minus 1pt}{4pt plus 1pt minus 1pt}
\usepackage{geometry}
\newif\ifrevisionblue
\revisionbluefalse
\newcommand{\rev}[1]{\ifrevisionblue\textcolor{blue}{#1}\else#1\fi}
\long\def\revised#1{\ifrevisionblue{\color{blue}#1}\else#1\fi}

\newtheorem{lemma}{Proposition }

\begin{document}
\bstctlcite{IEEEexample:BSTcontrol}

\title{DoRF++: Spherical Representation Learning over Doppler Radiance Fields for Robust Wi-Fi Sensing}

\author{Navid Hasanzadeh$^{1}$ and Shahrokh Valaee$^{1}$,~\IEEEmembership{Fellow,~IEEE}

\thanks{$^{1}$N. Hasanzadeh and S. Valaee are with the Department of Electrical \& Computer Engineering, University of Toronto, Toronto, ON, Canada. { navid.hasanzadeh@mail.utoronto.ca, valaee@ece.utoronto.ca.}}}

\maketitle

\begin{abstract}

Motivated by the IEEE 802.11bf effort to standardize advanced WLAN sensing, interest in Wi-Fi Channel State Information (CSI) for passive, device-free, and privacy-preserving activity and gesture recognition has grown rapidly. Recent studies have shown that Doppler velocity projections extracted from CSI, which directly reflect human-motion velocity, enable more robust human activity recognition (HAR) and stronger generalization across users and unseen conditions. Nevertheless, reliable generalization under real-world variability remains a major challenge that hinders the adoption of Wi-Fi sensing.
To address this challenge, we build on Doppler Radiance Fields (DoRF), adapting the concept of neural radiance fields (NeRF) from computer vision to Wi-Fi sensing. DoRF treats the extracted Doppler velocity projections as sparse virtual-camera views of human motion. From these views it recovers a regularized latent 3-D motion descriptor whose projections along learned effective Doppler directions explain the CSI observations. The recovered motion is then re-projected onto an equiangular grid of directions on the unit sphere, yielding a spherical representation of the underlying activity.
Because DoRF produces a signal defined on the sphere, this paper introduces DoRF++, a spherical representation-learning model that applies quadrature-aware spherical Transformers directly to the spherical field for activity classification. Experiments on our collected hand-gesture dataset show that DoRF++ noticeably outperforms state-of-the-art Wi-Fi-based HAR methods in cross-user generalization accuracy, especially for difficult gestures when only a single multi-antenna receiver access point (AP) is available.
\end{abstract}


\section{Introduction}
\IEEEPARstart{H}{uman} sensing technologies are evolving with the rise of smart environments and virtual reality applications, allowing seamless detection of and response to user actions. Wi-Fi-based human activity recognition (HAR) stands out due to its non-intrusive nature and broad accessibility~\cite{ahmad2024wifi}, particularly following the recent publication of the IEEE 802.11bf standard, which defines protocols for WLAN sensing in license-exempt bands from 1 to 7.125~GHz while remaining compatible with existing Wi-Fi devices~\cite{du2024overview}. Wi-Fi sensing also avoids the drawbacks of cameras, which require a clear line-of-sight path and record visual data, raising privacy concerns, and of wearables, which require continuous use and regular battery charging~\cite{radwan2025tutorial}, enabling indoor tracking in homes, offices, and hospitals for interactive and immersive applications.

\begin{figure}[!t]
	\centering

	\subfloat[]{%
		\includegraphics[width=0.7\linewidth]{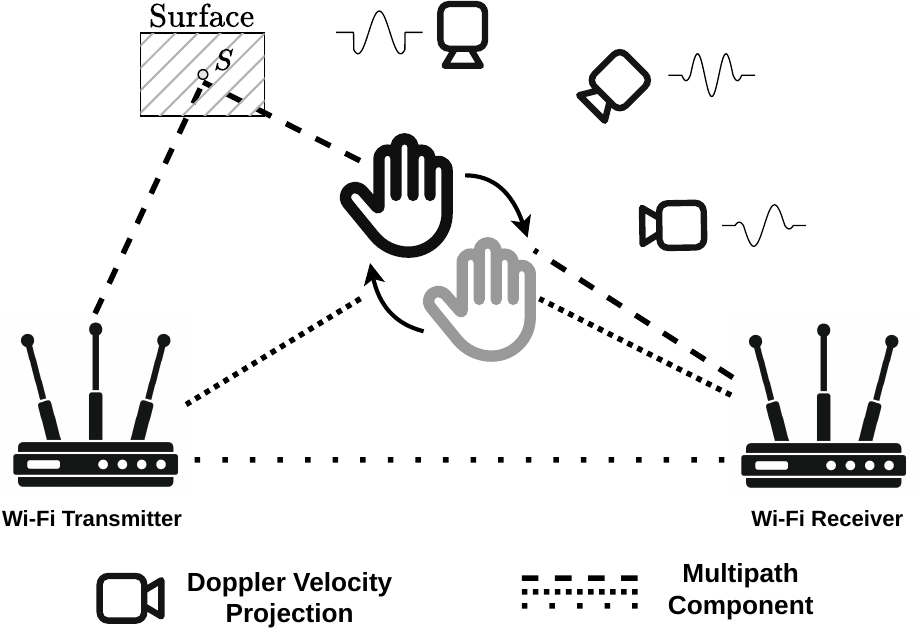}
		\label{figure:handcamera}
	}

	\vspace{0.5em}

	\subfloat[]{%
		\includegraphics[width=0.7\linewidth]{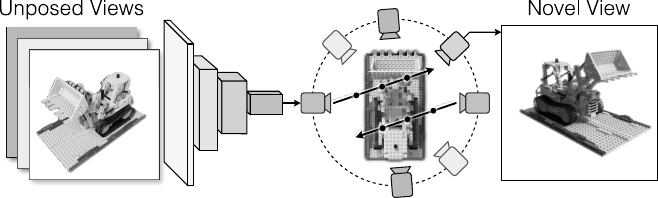}
		\label{figure:melon}
	}

\caption{
Comparison between the proposed DoRF for Wi-Fi sensing and NeRF in computer vision.
(a) DoRF represents an activity using multipath Doppler velocity projections from Wi-Fi CSI, capturing motion observations from multiple perspectives within the 3D environment.
(b) Unposed NeRF reconstructs a 3D volume from 2D images with unknown camera poses~\cite{levy2023melon}.
}
	\label{figure:dorf_nerf_comparison}
\end{figure}

Wi-Fi-based HAR infers human motion by analyzing how movement perturbs the wireless channel, observed through complex-valued Channel State Information (CSI). Prior work typically relies on either CSI magnitude or CSI phase~\cite{zandi2026beyond, yousefi2017survey, djogoHAR, jang2025study, zhang2024csi, peng2023rosefi}, often using feature extraction followed by machine learning, but these pipelines are usually reliable only in highly controlled and static settings. Magnitude-based methods degrade substantially when the environment or user changes, while phase-based approaches are further challenged by phase wrapping, hardware impairments, and multipath effects. Consequently, both families often generalize poorly to unseen users, locations, or devices, in some cases approaching chance-level accuracy.

An alternative line of work focuses on Doppler velocity extracted from Wi-Fi CSI, exploiting motion-induced frequency shifts to emphasize dynamic components while attenuating static structures. Several methods formulate velocity estimation as an angle-of-arrival problem, modeling the moving body as a single point source to recover directional motion cues~\cite{meneghello2022sharp, chen2022afall, zhang2021widar3}. However, these approaches compress multipath-rich observations at each access point (AP) into a single velocity time series. Because the underlying motion is three-dimensional and multipath interference can suppress activity-related signatures, this summary may be insufficient to distinguish activities with similar aggregate Doppler responses.

To move beyond single-estimate aggregation, the recent method MORIC~\cite{hasanzadeh2025moric} decomposes Wi-Fi CSI into delay-separated components and extracts a Doppler velocity projection from each one. These projections can be interpreted as observations of the same human motion from different unknown directions, analogous to one-dimensional virtual cameras induced by random multipath reflections. Treating them as an unordered and potentially repeated set of motion descriptors, aggregated by an order- and repetition-invariant classifier, substantially improves cross-user generalization over methods that rely on a single aggregated Doppler estimate.

However, MORIC does not explicitly model the geometry of these virtual-camera views. The projections are interpreted as separate motion descriptors, without estimating their viewing directions or organizing them into a coherent spatial representation. This is analogous to interpreting each piece of a puzzle independently, without assembling them into the complete scene. Moreover, Wi-Fi reflections are inherently random and strongly environment-dependent, so the effective viewpoints captured by each AP are sparse, limited, and may vary across trials and environments, and generalization can still degrade when test data contain motion views not well represented during training. These limitations motivate a geometry-aware and more comprehensive motion representation that is robust to Wi-Fi channel randomness and environmental reflections.

This work builds on \textbf{Do}ppler \textbf{R}adiance \textbf{F}ields (DoRF), a geometry-aware representation of Doppler observations that we introduced in preliminary form in earlier conference work and develop in full here\footnote{Preliminary versions of the Doppler Radiance Fields (DoRF) framework were accepted at \textit{IEEE CAMSAP~2025} and \textit{IEEE ICASSP~2026}~\cite{hasanzadeh2025dorf,hasanzadeh2026doppler}.}. Building on MORIC's extraction of multiple Doppler velocity projections, DoRF does not treat them as independent, unordered descriptors, but models them as sparse virtual-camera observations of a shared latent three-dimensional motion sequence.

Inspired by neural radiance fields (NeRF)~\cite{mildenhall2021nerf}, which learn a continuous three-dimensional scene representation from two-dimensional images, and unposed NeRF~\cite{levy2023melon}, which extends this idea to unknown camera poses, DoRF learns a latent motion sequence whose projections through learned effective Doppler directions explain the one-dimensional Doppler velocity signals extracted from Wi-Fi CSI. The recovered motion is then projected onto an equiangular grid of directions on the unit sphere, producing a comprehensive, geometry-aware Doppler radiance field with a consistent set of virtual viewpoints, as illustrated in Fig.~\ref{figure:dorf_nerf_comparison}. This paper uses alternating estimation to describe how the latent motion sequence and the effective Doppler projection vectors interact, following the perspective used in the original DoRF work~\cite{hasanzadeh2025dorf}. We also show that this unconstrained regularized factorization admits a closed-form global solution, given directly by the singular value decomposition of the observed Doppler matrix.

DoRF is a \emph{representation} that, for each receive antenna, recovers the latent three-dimensional motion sequence that explains the extracted Doppler velocity projections and re-projects it onto a fixed equiangular grid of directions on the unit sphere, producing an ordered spherical field. DoRF++, introduced in this paper, is the \emph{recognition model} built on that field. DoRF++ is motivated by the fact that a DoRF is a signal defined on the sphere, whereas the classifier previously used with DoRF treats the recovered descriptors as an unordered set, discarding the angular organization that DoRF provides. DoRF++ instead extracts a temporal feature vector per spherical direction, treats each direction as a token on the sphere, and processes these tokens with quadrature-aware spherical Transformers whose relative positional encoding depends only on the angle between directions.

The main contributions of this work are as follows:
\begin{itemize}
\item DoRF, presented here in full, is a representation of human motion built from multiple one-dimensional Doppler velocity projections extracted from Wi-Fi CSI. Treating each projection as a virtual camera induced by sparse multipath reflections, DoRF recovers a unified latent three-dimensional motion sequence and constructs a spherical multi-view description that is more robust to environmental and multipath variability. The formulation is accompanied by a closed-form spectral characterization of its global solution and an analysis of the latent-frame ambiguity of the recovered factors.
\item DoRF++, introduced in this paper, extends DoRF from a representation into a spherical representation-learning model. Rather than classifying the recovered motion descriptors with a generic order-invariant classifier, DoRF++ processes the synthesized DoRF directly in its spherical form. By computing nonlinear temporal features after directional mixing and relating them through spherical attention over known relative angles, it exploits the full angular coverage of the radiance field and fuses the receive-antenna-specific fields through permutation-invariant pooling.
\item To enhance CSI preprocessing, we introduce a common-RX phase sanitization method for multi-antenna Wi-Fi systems. By forming ratios between the CSI streams that reach the same receive antenna from two different transmit antennas, it more effectively suppresses receiver-side STO, SFO, and RF-chain mismatch, yielding cleaner phase dynamics.
\item An experimental evaluation on our collected challenging hand-gesture dataset shows substantial gains in cross-user generalization accuracy over state-of-the-art Wi-Fi-based HAR methods, and over DoRF combined with a generic classifier, using only a single receiver AP, particularly on complex gestures.
\end{itemize}

\section{Wi-Fi CSI Estimation and Doppler Velocity Extraction}\label{section:csi}
This section describes how the Doppler observations used by DoRF are obtained from raw Wi-Fi measurements through CSI estimation and the multipath channel model, a common-RX phase sanitization method that suppresses receiver-side synchronization errors, the effect of motion on the sanitized channel, and MUSIC-based extraction of one Doppler velocity projection per stream. Section~\ref{section:dorf} then organizes these projections into a Doppler radiance field, and Section~\ref{section:dorfpp} presents the DoRF++ spherical representation-learning model.
\subsection{Wi-Fi Channel State Information}

\noindent
In IEEE~802.11a/g/n/ac systems, each Wi-Fi frame contains a preamble with known training symbols called the \textit{Long Training Field} (LTF). Since these symbols are known in advance at the receiver and span the active OFDM subcarriers, they are used to estimate the wireless channel on each subcarrier. This per-subcarrier channel estimate is commonly referred to as CSI. For a transmitter with $A_t$ antennas and a receiver with $A_r$ antennas, the received signal vector at subcarrier frequency $f_c$ can be written as
\begin{equation}
\mathbf{x} = \mathbf{H}_{f_c}\mathbf{s}_{\rm tx} + \mathbf{n},
\end{equation}
where $\mathbf{x}\in\mathbb{C}^{A_r}$ is the received signal vector, $\mathbf{s}_{\rm tx}\in\mathbb{C}^{A_t}$ is the transmitted signal vector, $\mathbf{H}_{f_c}\in\mathbb{C}^{A_r\times A_t}$ is the channel matrix, and $\mathbf{n}$ denotes additive noise. Let $\mathbf{S}_{\mathrm{p}}\in\mathbb{C}^{A_t\times N_p}$ denote the pilot matrix, whose columns are the $N_p$ known transmitted pilot vectors, and let $\mathbf{Y}_{\mathrm{p}}=\mathbf{H}_{f_c}\mathbf{S}_{\mathrm{p}}+\mathbf{N}_{\mathrm{p}}\in\mathbb{C}^{A_r\times N_p}$ be the corresponding received pilot matrix, with $\mathbf{N}_{\mathrm{p}}$ the noise matrix. A Least Squares (LS) estimate of the CSI is then
\begin{equation}
\mathbf{H}_{f_c}^{\mathrm{LS}} = \mathbf{Y}_{\mathrm{p}}\mathbf{S}_{\mathrm{p}}^H(\mathbf{S}_{\mathrm{p}}\mathbf{S}_{\mathrm{p}}^H)^{-1},
\end{equation}
where $(\cdot)^H$ denotes the Hermitian transpose.

CSI is inherently shaped by multipath propagation, where reflections, scattering, and diffraction create multiple paths between transmit antenna $m$ and receive antenna $n$, and a common model expresses the channel as a sum over $L$ paths:
\begin{equation}
\label{eq:CSI_multipath}
H_{m,n}^{f_c}(s) = \sum_{l=1}^L \beta_{l}(s)\,e^{-\mathrm{j}2\pi d_{m,n,l}(s)\,f_c /c},
\end{equation}
where $H_{m,n}^{f_c}(s)$ is the CSI at time $s$, $\beta_l(s)\in\mathbb{C}$ is the complex gain of path $l$, $d_{m,n,l}(s)$ is its propagation distance, $c$ is the speed of light, and the exponential captures the phase shift due to propagation. Environmental changes or human motion perturb the multipath components, causing $\mathbf{H}_{f_c}$ to vary over time and enabling Wi-Fi sensing applications such as HAR.

\subsection{CSI Phase Sanitization via Common-RX Stream Ratio}

In real-world Wi-Fi deployments, CSI measurements are affected by multiple impairments arising from baseband processing, synchronization errors, and hardware and software nonidealities~\cite{ma2019wifi}. Accounting for these effects, the measured CSI, \(\hat{H}_{m,n}^{f_c}(s)\), at time \(s\) between transmit antenna \(m\) (TX) and receive antenna \(n\) (RX) can be expressed as
\begin{equation}
\begin{multlined}
\hat{H}_{m,n}^{f_c}(s)
=
\underbrace{\sum_{l=1}^{L}
\beta_l(s)\,
e^{-\mathrm{j}2\pi d_{m,n,l}(s)f_c/c}}_
{\substack{\text{Multipath}\\\text{channel}}}
\,
\underbrace{e^{-\mathrm{j}2\pi \tau_m(s)f_c}}_
{\substack{\text{Cyclic shift}\\\text{diversity}}}
\times
\\
\underbrace{e^{-\mathrm{j}2\pi \rho_n(s)f_c}}_
{\substack{\text{Sampling time}\\\text{offset}}}
\,
\underbrace{
e^{-\mathrm{j}2\pi \eta_n(s)
\left(\frac{f_c^\prime}{f_c}-1\right)f_c}}_
{\substack{\text{Sampling frequency}\\\text{offset}}}
\,
\underbrace{
q_{m,n}(s)\,
e^{-\mathrm{j}2\pi \zeta_{m,n}(s)}}_
{\text{Beamforming}},
\label{eq:CSI_channel}
\end{multlined}
\end{equation}
where \(\tau_m(s)\) is the cyclic shift diversity delay applied at transmit antenna \(m\), \(\rho_n(s)\) the sampling time offset (STO) of receive RF chain \(n\), and \(\eta_n(s)\) its sampling frequency offset (SFO), with \(f_c^\prime\) the actual subcarrier frequency, while \(q_{m,n}(s)\) and \(\zeta_{m,n}(s)\) capture the amplitude attenuation and phase shift introduced by beamforming.

To distinguish synchronization errors shared across the receiver from residual chain-dependent distortions, the STO and SFO terms can be decomposed as \(\rho_n(s)=\rho_{\mathrm{com}}(s)+\Delta\rho_n(s)\) and \(\eta_n(s)=\eta_{\mathrm{com}}(s)+\Delta\eta_n(s)\), where \(\rho_{\mathrm{com}}(s)\) and \(\eta_{\mathrm{com}}(s)\) denote the receiver-wide components, while \(\Delta\rho_n(s)\) and \(\Delta\eta_n(s)\) represent residual offsets specific to receive RF chain \(n\). This formulation accounts for practical RF-chain mismatch, imperfect calibration, and antenna-dependent timing or frequency offsets.

\begin{figure}[!t]
	\centering
	\subfloat{%
		\includegraphics[width=0.8\linewidth]{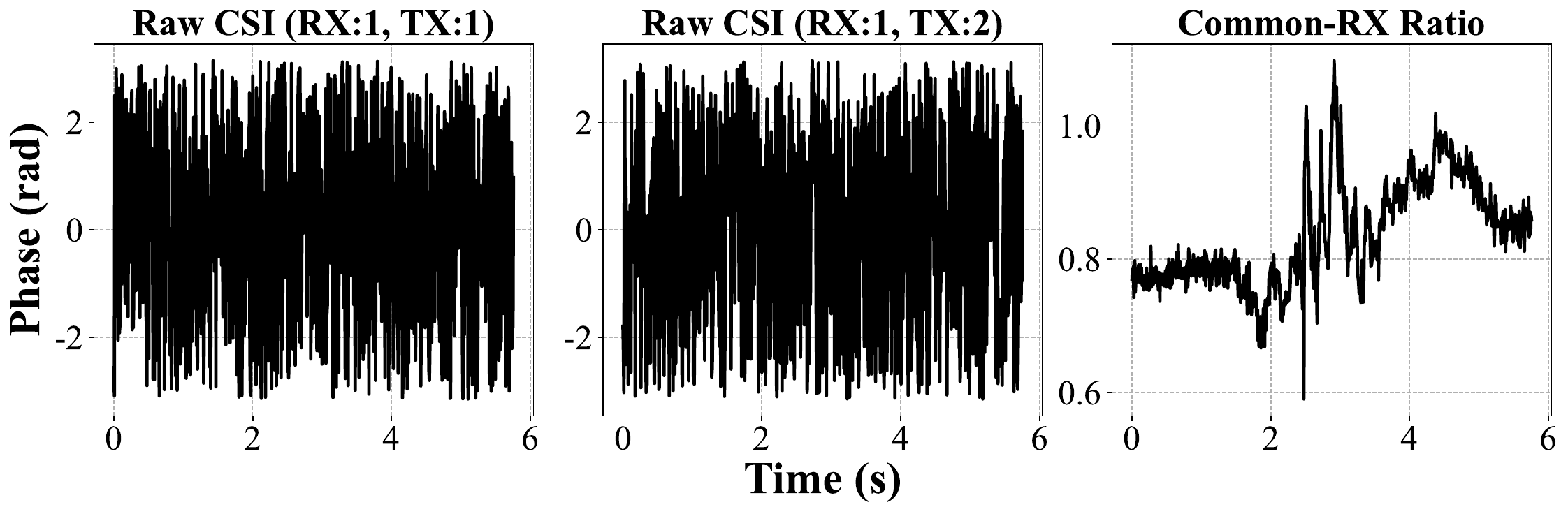}
	}
	\caption{Proposed CSI phase sanitization based on the ratio of common-RX, different-TX streams. The method effectively suppresses phase noise while revealing activity-related temporal patterns. The example shows one subcarrier measured during a circular hand gesture.}
\label{figure:noise}
\end{figure}

Among these impairments, STO and SFO have the most severe effect on CSI phase. Because Doppler information is encoded in fine-grained phase dynamics, these offsets can overshadow the motion-induced variations, whereas other imperfections contribute minor temporal perturbations. Simple temporal averaging or low-pass filtering is therefore insufficient and may even degrade the sensing signal, particularly with abrupt phase jumps, hardware-induced discontinuities, or bursty interference.

A common approach in prior work is CSI-ratio processing, in which two TX--RX streams sharing the same transmit antenna but using different receive antennas are divided to suppress common phase distortions~\cite{wu2022wifi}. This implicitly assumes that the corresponding receive RF chains are well matched, which practical multi-chain receivers do not guarantee because of RF-chain mismatch, imperfect calibration, and antenna-dependent offsets. In terms of the decomposition above, a conventional different-RX ratio cancels the common components \(\rho_{\mathrm{com}}(s)\) and \(\eta_{\mathrm{com}}(s)\), but retains differential residual terms determined by \(\Delta\rho_{n_1}(s)-\Delta\rho_{n_2}(s)\) and \(\Delta\eta_{n_1}(s)-\Delta\eta_{n_2}(s)\). It may therefore not only fail to fully suppress phase noise but also introduce additional residual phase errors, which are particularly destructive to Doppler sensing~\cite{102859317}.

To address these limitations, this work introduces a common-RX CSI ratio for multi-antenna Wi-Fi transmitters, a configuration widely available in practical deployments. Unlike conventional CSI-ratio methods, which form ratios between streams associated with the same transmit antenna, the proposed ratio \(H_{m_1,m_2,n}^{f_c}(s)\) is constructed from two streams received by the same antenna \(n\) but transmitted from two distinct antennas:
\begin{equation}
H_{m_1,m_2,n}^{f_c}(s)
=
\frac{\hat{H}_{m_1,n}^{f_c}(s)}
{\hat{H}_{m_2,n}^{f_c}(s)},
\label{eq:csi_ratio_proposed}
\end{equation}
Here and throughout, a \emph{transmit stream} denotes the measured CSI time series \(\hat{H}_{m,n}^{f_c}(s)\) of one transmit--receive antenna pair, that is, the channel seen at receive antenna \(n\) for the signal radiated by transmit antenna \(m\). This time series already superimposes all propagation paths between that pair rather than a single path. In~\eqref{eq:csi_ratio_proposed}, both streams are measured through the same receive RF chain and therefore experience the same receiver-wide and chain-dependent STO and SFO distortions, which cancel when the ratio is formed. Unlike a different-RX ratio, the common-RX construction therefore does not retain differential receiver-chain synchronization errors and more effectively reduces residual phase distortion. As illustrated in Fig.~\ref{figure:noise}, it suppresses dominant phase noise while preserving the fine-grained temporal structure associated with human activity.

\subsection{Effect of Motion on Wi-Fi CSI}

Indoor Wi-Fi propagation produces multiple concurrent paths between a transmitter and a receiver, including a line-of-sight component and numerous reflections from surrounding surfaces, which allow the underlying velocity \(\mathbf{v}\) to be observed from multiple effective directions. When the contributions associated with a path are concentrated around a dominant direction, the resulting Doppler response can be approximated as a linear projection of \(\mathbf{v}\) onto that direction. When the reflections are spread over several directions, the response instead represents a mixture of motion projections. Although such an observation appears blurred when interpreted as a single projection, it still contains valuable information about the underlying 3D motion, analogous to a motion-blurred image, which superimposes several viewpoints rather than one sharp view. Each resolvable path, or unresolved group of paths, can therefore be interpreted as a virtual one-dimensional camera that observes the same 3D motion from a different effective direction and with a path-dependent degree of projection clarity. A detailed derivation of this projection model and its assumptions is provided in MORIC~\cite{hasanzadeh2025moric}.

For each effective propagation component \(p\), representing either a resolvable path or an unresolved group of paths, the motion-induced path-length change can be approximated by projecting the true three-dimensional velocity \(\mathbf{v}\) onto the component's effective observation direction. Let \(\bar{\mathbf{r}}_p\in\mathbb{R}^3\) denote that unit direction, let \(g_p\geq 0\) denote its effective projection gain, which captures propagation geometry, multipath visibility, and the clarity of the corresponding Doppler observation, and define \(\mathbf{r}_p=g_p\bar{\mathbf{r}}_p\). With \(\theta_p\) the angle between \(\mathbf{v}\) and \(\bar{\mathbf{r}}_p\), and over a sufficiently short interval \(t\) during which the velocity, observation geometry, and path gain are approximately constant, the resulting path-length variation is
\begin{equation}
\Delta L_p
\approx
\mathbf{v}^{\top}\mathbf{r}_p\,t
=
g_p\|\mathbf{v}\|\,t\,\cos\theta_p
=
v_p t,
\label{eq:total_delay_eq}
\end{equation}
where \(v_p=\mathbf{v}^{\top}\mathbf{r}_p=g_p\|\mathbf{v}\|\cos\theta_p\) is the Doppler velocity projection associated with component \(p\). For unresolved groups containing reflections from several directions, this effective projection is an approximation, with deviations absorbed into the modeling residual. Each component thus provides a distinct one-dimensional observation of the same underlying 3D velocity, with magnitude and sign determined by the effective projection gain and the alignment between the motion and the observation direction. Components aligned with the motion produce stronger projections, while nearly orthogonal ones produce weaker projections.

The corresponding delay variation is \(\Delta \tau_p=\Delta L_p/c\), which induces a phase rotation at carrier frequency \(f_c\). Accordingly, the channel contribution associated with component \(p\), between transmit antenna \(m\) and receive antenna \(n\), evolves approximately as \({H}_{p,m,n}^{f_c}(s+t)\approx{H}_{p,m,n}^{f_c}(s)\,e^{-\mathrm{j}2\pi f_c \Delta \tau_p}\).
Since the measured CSI is the coherent superposition of multiple propagation components, the complete channel can be expressed as
\begin{equation}
{H}_{m,n}^{f_c}(s+t)
\approx
\sum_p
{H}_{p,m,n}^{f_c}(s)\,
e^{-\mathrm{j}2\pi f_c \Delta \tau_p}.
\end{equation}
Different components may therefore exhibit different temporal phase-rotation rates, corresponding to different Doppler projections of the same underlying velocity. These complementary projections can be exploited for motion tracking and sensing applications such as human activity and gesture recognition.

\subsection{MUSIC-Based Doppler Extraction from Wi-Fi CSI}

Wi-Fi CSI is inherently a multipath measurement. Each OFDM subcarrier captures a superposition of propagation components with different delays and gains. This suggests delay--Doppler decomposition, which separates CSI by propagation delay and extracts motion-induced Doppler components within each delay bin, yielding a large set of observations of the underlying motion.

Delay--Doppler decomposition, as used in prior work such as MORIC, extracts Doppler information independently from each delay bin, which has two main limitations. First, some delay bins may be very noisy, leading to unstable Doppler estimates and many outlier projections. Second, at higher carrier frequencies such as 5~GHz, the phase distortion caused by a given timing error scales with \(f_c\), so small timing offsets and hardware imperfections introduce larger phase errors that propagate into the Doppler estimates. Delay-wise Doppler extraction therefore tends to be noisier and less stable at 5~GHz than at 2.4~GHz.

Rather than relying on delay-wise estimation, this work assumes Wi-Fi transmitters equipped with multiple antennas and estimates Doppler independently from each common-RX, different-TX CSI-ratio stream by analyzing its temporal variation using the MUSIC algorithm~\cite{schmidt1986multiple}. The number of resulting Doppler observations is determined by the available transmit-antenna pairs per receive antenna, with $N=6$ streams per receive antenna in our setup. Although this yields fewer observations in our setup compared to methods like MORIC, multiple transmit-antenna pairs still provide an informative set of projections, and each projection is estimated from the temporal evolution of a complete CSI-ratio stream rather than from a potentially noisy delay bin, leading to a more robust representation of human motion and improved resilience to noise.

The CSI ratio of two multipath channels is not, in general, exactly represented by a single physical Doppler component. Its temporal phase evolution reflects the difference between the motion-induced phase variations of its two constituent transmit streams, that is, of the per-transmit-antenna channels \(\hat{H}_{m_1,n}^{f_c}(s)\) and \(\hat{H}_{m_2,n}^{f_c}(s)\) measured at the same receive antenna \(n\). Over a sufficiently short time interval, however, this evolution can be locally approximated as a superposition of effective Doppler components. Under a linear-projection interpretation, the differential response of the two channels can be written as
\begin{equation}
v_{m_1,m_2,n}(s)
\approx
\mathbf{v}(s)^\top
\widetilde{\mathbf{r}}_{m_1,m_2,n}
+
\varepsilon_{m_1,m_2,n}(s),
\end{equation}
where \(\widetilde{\mathbf{r}}_{m_1,m_2,n}=\mathbf{r}_{m_1,n}-\mathbf{r}_{m_2,n}\), the vectors \(\mathbf{r}_{m_1,n}\) and \(\mathbf{r}_{m_2,n}\) are the unknown effective Doppler vectors of the two transmit streams, and \(\varepsilon_{m_1,m_2,n}(s)\) captures noise and modeling residuals. The ratio stream can therefore be interpreted as an observation of the underlying motion from a single unknown effective direction. The constituent directions are induced by the multipath environment. We do not require the constituent directions to be recovered individually, because their differential effect is absorbed into a single stream-specific effective Doppler vector, and each CSI-ratio stream is treated as an observation from an unknown, environment-dependent effective direction.

Specifically, each CSI-ratio stream is divided into overlapping fixed-size windows, and one effective Doppler velocity is estimated from each. Let \(i=(m_1,m_2,n)\) index a stream. The samples within a window of length \(W\) starting at time \(s\) are stacked as \(\mathbf{h}_{i}(s)=[H_{i}^{f_c}(s),\,\ldots,\,H_{i}^{f_c}(s+(W-1)\Delta t)]^\top\), where \(H_i^{f_c}(s)\) is the CSI ratio in~\eqref{eq:csi_ratio_proposed} and \(\Delta t\) the CSI sampling interval. Within each window, the motion and effective propagation geometry are assumed to vary slowly enough that the Doppler frequencies are approximately constant.

Assuming \(L\) dominant effective motion components with distinct Doppler frequencies, the signal can be locally approximated as
\begin{equation}
\mathbf{h}_{i}(s)
\approx
\sum_{l=1}^{L}
\alpha_{i,l}\,\mathbf{a}(f_{i,l})
+
\mathbf{n}_{i},
\label{eq:effective_ratio_doppler_model}
\end{equation}
where \(\alpha_{i,l}\) is the complex coefficient of the \(l\)-th effective component, \(\mathbf{n}_{i}\) denotes noise and modeling residuals, and the Doppler steering vector is \(\mathbf{a}(f)=[1,\,e^{\mathrm{j}2\pi f \Delta t},\,\ldots,\,e^{\mathrm{j}2\pi f (W-1)\Delta t}]^\top\).

For each stream, the covariance matrix is estimated from multiple snapshots across the available OFDM subcarriers. Within each window of length \(W\), the samples from subcarrier \(k\) form a snapshot \(\mathbf{h}_{i,k}(s)=[H_{i,k}(s),\,\ldots,\,H_{i,k}(s+(W-1)\Delta t)]^\top\in\mathbb{C}^{W}\). Using the \(N_{\rm sc}\) subcarriers as snapshots, the
covariance matrix is estimated as
\begin{equation}
\mathbf{C}_{i}(s)
=
\frac{1}{N_{\rm sc}}
\sum_{k=1}^{N_{\rm sc}}
\mathbf{h}_{i,k}(s)\mathbf{h}_{i,k}^{H}(s)
\in\mathbb{C}^{W\times W}.
\end{equation}

Since the occupied bandwidth is small relative to the carrier frequency, the motion-induced Doppler frequencies are assumed approximately common across OFDM subcarriers within each temporal window, while subcarrier-dependent propagation effects are absorbed into their complex coefficients. The MUSIC pseudo-spectrum is then computed as
\begin{equation}
P_{\mathrm{MUSIC}}^{(i)}(f)
=
\frac{1}{
\mathbf{a}^{H}(f)
\mathbf{U}_{\mathrm{N},i}
\mathbf{U}_{\mathrm{N},i}^{H}
\mathbf{a}(f)
},
\label{music}
\end{equation}
where \(\mathbf{a}(f)\in\mathbb{C}^{W}\) is the Doppler steering vector. The dominant Doppler frequency is converted to a Doppler velocity projection as
\begin{equation}
v_r=\lambda f,
\qquad
\lambda=\frac{c}{f_c}.
\label{v_r_music}
\end{equation}
For a CSI-ratio stream, \(v_r\) represents the effective differential path-length rate encoded by the relative phase evolution of the two transmit streams and can be treated directly as a one-dimensional Doppler projection of the underlying motion.

Because the MUSIC spectrum can contain multiple peaks in a given window, we retain only the dominant peak for each stream and time index:
\begin{equation}
f_i^\star(s)
=
\arg\max_{f\in\mathcal{F}}
P_{\mathrm{MUSIC}}^{(s,i)}(f),
\qquad
v_r(s;i)
=
\lambda f_i^\star(s),
\end{equation}
where \(\mathcal{F}\) is the searched Doppler grid. This is applied independently to each CSI-ratio stream constructed from transmit-antenna pairs sharing the same receive antenna, yielding a set of complementary Doppler projections.

\section{Doppler Radiance Fields}\label{section:dorf}

Actual human-motion velocity occurs in three dimensions, while each sanitized CSI-ratio stream provides only a one-dimensional Doppler velocity projection. DoRF combines these projections into a latent three-dimensional underlying motion descriptor, which can ideally provide a more complete and robust description of the actual motion. Doppler velocity projections extracted from Wi-Fi CSI are sparse, unordered, and environment-dependent, as the distribution of dominant effective reflections is shaped by the surrounding scene. The underlying multipath geometry determines both the amount of motion information captured by each CSI-ratio stream and the effective direction from which the motion is observed, so the available projections often cover only a limited and non-uniform portion of the angular space. This can limit generalization when test conditions involve motion viewpoints that were weakly represented during training.

All Doppler projections obtained at a given receive antenna are views of the same three-dimensional motion. Recovering that motion therefore allows the sparse, unevenly distributed observations to be re-expressed along a fixed set of directions covering the whole sphere, so that the resulting description is comparable across trials and environments. For each receive antenna, DoRF is constructed independently from the common-RX, different-TX CSI-ratio streams associated with that antenna. To simplify the notation, the receive-antenna index is omitted in the following derivation, and $i=1,\ldots,N$ indexes the corresponding transmit-antenna pairs.

Each Doppler projection is a one-dimensional observation of hand motion from an unknown effective direction and does not usually correspond to a single physical path. Instead, it is formed by the relative motion-induced phase evolution of two transmit streams, each containing the superposition of many unresolved propagation paths and scatterers, and is therefore subject to the blurring discussed above. Different CSI-ratio streams nevertheless provide complementary motion views. One may primarily capture left--right motion, another forward--backward motion, and another a mixture of vertical and horizontal components. Unlike optical cameras, these effective viewing directions are not known a priori, because they are induced by the transmitter--receiver geometry, antenna configuration, and surrounding reflectors.

This motivates representing each CSI-ratio stream by a stream-specific effective Doppler projection vector in $\mathbb{R}^3$. Under the standard far-field or small-motion approximation, the Doppler response depends primarily on the effective observation direction, while the strength and reliability of that projection vary across streams because of propagation geometry, multipath visibility, and Doppler-estimation quality.

Let $\mathbf{r}_i\in\mathbb{R}^3$ denote the effective Doppler projection vector associated with the $i$-th Doppler projection. The corresponding Doppler velocity is modeled as
\begin{equation}
 v_r(s;i)=\mathbf{v}(s)^\top\mathbf{r}_i+\varepsilon(s,i),
 \label{eq:dorf_projection_model}
\end{equation}
where $\mathbf{v}(s)\in\mathbb{R}^3$ is a latent three-dimensional velocity descriptor representing the aggregate hand motion, and $\varepsilon(s,i)$ captures measurement noise, unresolved multi-scatterer effects, and modeling residuals. The direction of $\mathbf{r}_i$ is the effective motion-observation direction and its norm the stream-dependent projection gain. When a unit viewing direction is needed, we define $\bar{\mathbf r}_i=\mathbf r_i/\|\mathbf r_i\|$.

Let $\mathbf{V}_r\in\mathbb{R}^{T\times N}$ denote the Doppler velocity projections over $T$ time steps and $N$ CSI-ratio streams, where each element $\mathbf{V}_r(s,i)=v_r(s;i)$ is obtained using the dominant-peak MUSIC-based Doppler velocity estimation procedure in \eqref{music} and \eqref{v_r_music}. Stacking all time steps and streams yields the rank-three-plus-residual model
\begin{equation}
 \mathbf{V}_r=\mathbf{V}\mathbf{R}+\mathbf{E},
 \label{eq:dorf_low_rank_model}
\end{equation}
where $\mathbf{V}\in\mathbb{R}^{T\times 3}$ contains $\mathbf{v}(s)^\top$ as rows, $\mathbf{R}\in\mathbb{R}^{3\times N}$ contains the effective Doppler projection vectors $\mathbf{r}_i$ as columns, and $\mathbf{E}\in\mathbb{R}^{T\times N}$ collects measurement noise and modeling residuals.

Both the latent velocity matrix $\mathbf{V}$ and the effective projection matrix $\mathbf{R}$ are unknown and must be estimated from the observed matrix $\mathbf{V}_r$. We jointly seek a latent 3D velocity sequence and effective Doppler projection vectors whose projections reproduce the observed Doppler velocity matrix. This work formulates their joint estimation as an unconstrained regularized matrix factorization
\begin{equation}
\min_{\mathbf V,\mathbf R}
\frac{1}{2TN}\|\mathbf V_r-\mathbf V\mathbf R\|_F^2
+
\frac{\mu}{2T}\|\mathbf V\|_F^2
+
\frac{\gamma}{2N}\|\mathbf R\|_F^2,
\label{eq:objective_re}
\end{equation}
inspired by the shared-representation and unknown-viewpoint decomposition underlying unposed NeRF. Here the latent motion $\mathbf{V}$ acts as a shared representation constrained to be consistent across all views, and the unknown columns of $\mathbf{R}$ act as learned effective Doppler projection vectors, with each measurement $\mathbf{V}_r(s,i)$ modeled as the projection of the shared latent motion onto its associated effective vector. The reconstruction term enforces consistency between the observed and reconstructed projections, while the Frobenius regularizers stabilize the factorization under noise and incomplete angular coverage.

\subsection{Alternating Estimation}
\begin{algorithm}[!t]
\caption{Alternating Estimation of DoRF Motion and Effective Projections}
\label{alg:alt-opt-3d-velocity}
\setlength{\linewidth}{\dimexpr\columnwidth-\algomargin\relax}
\begin{algorithmic}[1]
\setlength{\abovedisplayskip}{0pt}
\setlength{\belowdisplayskip}{0pt}
\setlength{\abovedisplayshortskip}{0pt}
\setlength{\belowdisplayshortskip}{2pt}
\REQUIRE Doppler projections $\mathbf{V}_r\in\mathbb{R}^{T\times N}$, tolerance $\epsilon_{\mathrm{tol}}$, regularizers $\mu,\gamma$, maximum number of iterations
\STATE Compute a rank-three truncated singular value decomposition $\mathbf{V}_r\approx\mathbf{U}_3\boldsymbol{\Sigma}_3\boldsymbol{\Psi}_3^\top$ and initialize
\[
\mathbf{R}\leftarrow\boldsymbol{\Sigma}_3^{1/2}\boldsymbol{\Psi}_3^\top
\in\mathbb{R}^{3\times N}.
\]
\STATE Initialize $k\leftarrow0$
\REPEAT
  \STATE $k\leftarrow k+1$
  \STATE \textbf{Velocity update:}
  \[
  \mathbf{V} \leftarrow 
  \mathbf{V}_r \mathbf{R}^\top
  \left(
  \mathbf{R}\mathbf{R}^\top+\mu N\mathbf{I}_3
  \right)^{-1}
  \in\mathbb{R}^{T\times 3}
  \]
  \STATE \textbf{Direction update:}
  \[
  \mathbf{R}\leftarrow
  \left(
  \mathbf{V}^\top \mathbf{V}+\gamma T\mathbf{I}_3
  \right)^{-1}
  \mathbf{V}^\top \mathbf{V}_r
  \in\mathbb{R}^{3\times N}
  \]
  \STATE \textbf{Prediction:}
  \[
  \widehat{\mathbf{V}}_r\leftarrow \mathbf{V} \mathbf{R}
  \]
  \STATE \textbf{Loss:}
  \[
  \mathcal{L}^{(k)}
  =
  \tfrac{1}{2TN}
  \|\widehat{\mathbf{V}}_r-\mathbf{V}_r\|_F^{2}
  +
  \tfrac{\mu}{2T}
  \|\mathbf{V}\|_F^{2}
  +
  \tfrac{\gamma}{2N}
  \|\mathbf{R}\|_F^{2}
  \]
\vspace{0.35em}
\UNTIL{$k>1$ and $\left|\mathcal{L}^{(k)}-\mathcal{L}^{(k-1)}\right|/\mathcal{L}^{(k-1)}<\epsilon_{\mathrm{tol}}$} or maximum iterations

\STATE \textbf{Unit-direction extraction:}
\[
\bar{\mathbf r}_i\leftarrow
\frac{\mathbf r_i}{\|\mathbf r_i\|},
\qquad i=1,\dots,N,
\]
and let
\[
\bar{\mathbf R}
\leftarrow
[\bar{\mathbf r}_1,\dots,\bar{\mathbf r}_N].
\]

\STATE \textbf{DoRF construction:} let $\{\mathbf{d}_{uw}\}_{u=0,w=0}^{M-1,2M-1}$ be a fixed equiangular grid on the unit sphere, and compute
\[
\begin{gathered}
\mathbf{P}(s,u,w)\leftarrow \mathbf{v}(s)^\top\mathbf{d}_{uw},\\
\mathbf{P}\in\mathbb{R}^{T\times M\times 2M}.
\end{gathered}
\]
\RETURN $\mathbf{P}\in\mathbb{R}^{T\times M\times 2M}$, $\mathbf{V}\in\mathbb{R}^{T\times 3}$, $\mathbf{R}\in\mathbb{R}^{3\times N}$, $\bar{\mathbf R}\in\mathbb{R}^{3\times N}$
\end{algorithmic}
\end{algorithm}

The original DoRF work~\cite{hasanzadeh2025dorf} estimated the latent motion and the viewing directions by alternating between them. We adopt the same alternating formulation to describe how the latent motion sequence and the effective Doppler projection vectors jointly explain the observed Doppler projections. Algorithm~\ref{alg:alt-opt-3d-velocity} states the corresponding alternating ridge updates for the present objective and should be distinguished from the normalized-direction procedure of~\cite{hasanzadeh2025dorf}, whose constrained steps are not solved unchanged by \eqref{eq:svd_solution}.

The alternating updates expose the coupling between the latent motion sequence and the effective Doppler projection vectors. With the effective projection matrix fixed, the motion update finds the latent sequence that best explains all streams jointly. With the latent sequence fixed, the projection update estimates how each stream is represented in that latent space. For the unconstrained objective in \eqref{eq:objective_re}, both steps are ridge-regression problems with the unique minimizers
\begin{equation}
\begin{aligned}
\mathbf{V}&\leftarrow\mathbf{V}_r\mathbf{R}^{\top}\bigl(\mathbf{R}\mathbf{R}^{\top}+\mu N\mathbf{I}_3\bigr)^{-1},\\
\mathbf{R}&\leftarrow\bigl(\mathbf{V}^{\top}\mathbf{V}+\gamma T\mathbf{I}_3\bigr)^{-1}\mathbf{V}^{\top}\mathbf{V}_r,
\end{aligned}
\label{eq:ridge_updates}
\end{equation}
so the scheme is block coordinate descent~\cite{wei2012doa} on \eqref{eq:objective_re}. Each exact block update minimizes the objective over one block with the other fixed, so the objective value is non-increasing across iterations and, being bounded below, converges. The objective is convex in each block but not jointly convex, because of the bilinear product $\mathbf{V}\mathbf{R}$, so a finite iteration budget need not attain the global minimum, whereas the closed-form construction in \eqref{eq:svd_solution} supplies one directly. Algorithm~\ref{alg:alt-opt-3d-velocity} summarizes the alternating procedure, and the column normalization producing $\bar{\mathbf R}$ remains a post-processing step outside the minimization.

\subsection{Closed-Form SVD Solution}
The alternating algorithm describes the coupled estimation, while the following SVD construction gives a direct closed-form solution and characterizes the global structure of the unconstrained objective.

The unconstrained formulation in \eqref{eq:objective_re} is a regularized rank-at-most-three factorization of the observed Doppler matrix, and it admits a closed-form global solution. Scaling \eqref{eq:objective_re} by $TN$ leaves its minimizers unchanged and turns the penalties into $\tfrac{\lambda_{V}}{2}\|\mathbf{V}\|_F^2+\tfrac{\lambda_{R}}{2}\|\mathbf{R}\|_F^2$ with
\begin{equation}
\lambda_{V}=\mu N,
\qquad
\lambda_{R}=\gamma T,
\qquad
\tau_{\mathrm{sv}}=\sqrt{\lambda_{V}\lambda_{R}}=\sqrt{\mu\gamma TN}.
\label{eq:svd_consts}
\end{equation}
For a fixed product $\mathbf{Z}=\mathbf{V}\mathbf{R}$, the balanced-factor form of the nuclear norm gives
\[
\min_{\mathbf{V}\mathbf{R}=\mathbf{Z}}
\Bigl\{\tfrac{\lambda_{V}}{2}\|\mathbf{V}\|_F^2+\tfrac{\lambda_{R}}{2}\|\mathbf{R}\|_F^2\Bigr\}
=\tau_{\mathrm{sv}}\|\mathbf{Z}\|_*,
\]
so \eqref{eq:objective_re} is equivalent to the product-level problem
\begin{equation}
\min_{\operatorname{rank}(\mathbf{Z})\leq 3}
\tfrac{1}{2}\|\mathbf{V}_r-\mathbf{Z}\|_F^2+\tau_{\mathrm{sv}}\|\mathbf{Z}\|_*.
\label{eq:product_level}
\end{equation}
Let $\mathbf{V}_r=\mathbf{U}\boldsymbol{\Sigma}\boldsymbol{\Psi}^\top$ be a singular value decomposition with $\sigma_1\geq\sigma_2\geq\cdots$, let $\mathbf{U}_3$ and $\boldsymbol{\Psi}_3$ collect the three leading left and right singular vectors, and define
\begin{equation}
\begin{aligned}
\widetilde{\sigma}_k&=(\sigma_k-\tau_{\mathrm{sv}})_+,\\
\widetilde{\boldsymbol{\Sigma}}_3&=\operatorname{diag}(\widetilde{\sigma}_1,\widetilde{\sigma}_2,\widetilde{\sigma}_3),\\
\kappa_{\mathrm{bal}}&=\Bigl(\tfrac{\gamma T}{\mu N}\Bigr)^{1/4}.
\end{aligned}
\label{eq:svd_shrink}
\end{equation}
A global minimizer of \eqref{eq:objective_re} is then obtained directly as
\begin{equation}
\mathbf{V}^\star=\kappa_{\mathrm{bal}}\,\mathbf{U}_3\widetilde{\boldsymbol{\Sigma}}_3^{1/2},
\qquad
\mathbf{R}^\star=\kappa_{\mathrm{bal}}^{-1}\widetilde{\boldsymbol{\Sigma}}_3^{1/2}\boldsymbol{\Psi}_3^\top,
\label{eq:svd_solution}
\end{equation}
with product $\mathbf{V}^\star\mathbf{R}^\star=\mathbf{U}_3\widetilde{\boldsymbol{\Sigma}}_3\boldsymbol{\Psi}_3^\top$. Indeed, for fixed singular values, aligning $\mathbf{Z}$ with the singular vectors of $\mathbf{V}_r$ minimizes the reconstruction term. Each retained value then minimizes $\tfrac{1}{2}(\sigma_k-z)^2+\tau_{\mathrm{sv}} z$ over $z\geq0$, giving the soft threshold $\widetilde{\sigma}_k$. Finally, \eqref{eq:svd_solution} attains the minimum balanced penalty $\tau_{\mathrm{sv}}\sum_k \widetilde{\sigma}_k$ for that product. Components suppressed by the threshold are kept as zero entries so that the factors retain their three-dimensional shapes, and with $T=500$ and $N=6$ at least three singular triplets are always available.

This construction determines the latent motion descriptor and the effective projection vectors at once, without iterating between them, and shows the role of each regularizer, where $\tau_{\mathrm{sv}}$ sets the singular-value shrinkage and $\kappa_{\mathrm{bal}}$ fixes how scale is split between the latent motion sequence and the effective Doppler projection vectors. The optimal product is unique unless an active tie crosses the rank cutoff, that is, unless $\sigma_3=\sigma_4>\tau_{\mathrm{sv}}$. It is unique when $\sigma_3>\sigma_4$ or when $\sigma_3\leq\tau_{\mathrm{sv}}$, and repeated values lying entirely inside the retained subspace do not by themselves affect uniqueness of the product. For a fixed optimal product the balanced factors are determined only up to a common orthogonal transformation, including reflection, as discussed in Section~\ref{subsec:ambiguity}. Column normalization of $\mathbf{R}$ remains optional post-processing used only for directional visualization, where a column of zero norm has no defined direction and is omitted.

\subsection{Latent Coordinates and Spherical Re-Projection}

The recovered sequence $\mathbf{V}\in\mathbb{R}^{T\times 3}$ represents the latent three-dimensional motion in a coordinate system induced by the available Doppler projections. Random multipath propagation provides no fixed reference frame anchoring the recovered axes to the physical environment, and the scale of $\mathbf{V}$ is determined jointly with that of $\mathbf{R}$ through the regularizers. As a result, $\mathbf{V}$ is expressed in an unknown coordinate frame and scale that may vary across trials, so direct cross-trial comparisons require additional alignment. The descriptor is therefore latent rather than calibrated, because its axes and magnitudes are set by the regularized factorization, and recovered angles and scales do not correspond to metric properties of the physical motion.

Furthermore, the normalized recovered directions $\{\bar{\mathbf r}_i\}$ inherit the original observation geometry and remain non-uniformly distributed.

To obtain an ordered spherical representation within the recovered latent coordinate frame, each $\mathbf{v}(s)$ is re-projected onto a fixed equiangular grid of directions on the unit sphere. For an $M\times 2M$ latitude--longitude grid, define
\begin{align}
\theta_{u} &= \frac{(u+0.5)\pi}{M}, \quad u = 0,1,\dots,M-1, \\
\phi_{w} &= \frac{(w+0.5)2\pi}{2M}, \quad w = 0,1,\dots,2M-1,
\end{align}
and the corresponding unit vectors $\mathbf{d}_{uw}=[\sin\theta_u\cos\phi_w,\ \sin\theta_u\sin\phi_w,\ \cos\theta_u]^\top$. For each time index $s$, the radial component
\begin{equation}
\mathbf{P}(s,u,w)=\mathbf{v}(s)^\top\mathbf{d}_{uw}
\end{equation}
forms the Doppler radiance field $\mathbf{P}\in\mathbb{R}^{T\times M\times 2M}$, which remains subject to the global rotation or reflection ambiguity of the recovered latent coordinate system.

Although it adds no measurement information and stays fully determined by $\mathbf{V}$, this full-sphere equiangular re-projection with area-aware quadrature changes the information available to the learning stage. It provides temporal projections over the whole latent sphere rather than along three selected axes, and a global orthogonal change of the latent frame, $\mathbf{V}\mapsto\mathbf{V}\mathbf{Q}$, induces the reorientation $\mathbf{P}'(s,\mathbf{d})=\mathbf{P}(s,\mathbf{Q}\mathbf{d})$ of the continuous field, which is a structured transformation matched to spherical processing rather than an alignment to any physical frame.

The directional expansion is consequential because the temporal features are nonlinear. Each direction mixes the latent components as $\mathbf{P}(s,\mathbf{d})=d_x v_x(s)+d_y v_y(s)+d_z v_z(s)$ before the shared summaries are computed, so maximum and proportion-of-positive-value statistics of these mixtures are in general not linear combinations of the same statistics taken along the three latent axes. The relative angles $\mathbf{d}_j^\top\mathbf{d}_{j'}$ then supply a geometric prior for combining the resulting features through spherical attention. These properties follow from the representation and the architecture, and add no measurement information. The equiangular latitude--longitude grid represents unequal spherical areas, which the quadrature weights of Section~\ref{section:dorfpp} account for. Fig.~\ref{figure:Dorf} shows an example DoRF.

\begin{figure}[!t]
    \centering
    \subfloat{%
        \includegraphics[width=0.8\linewidth]{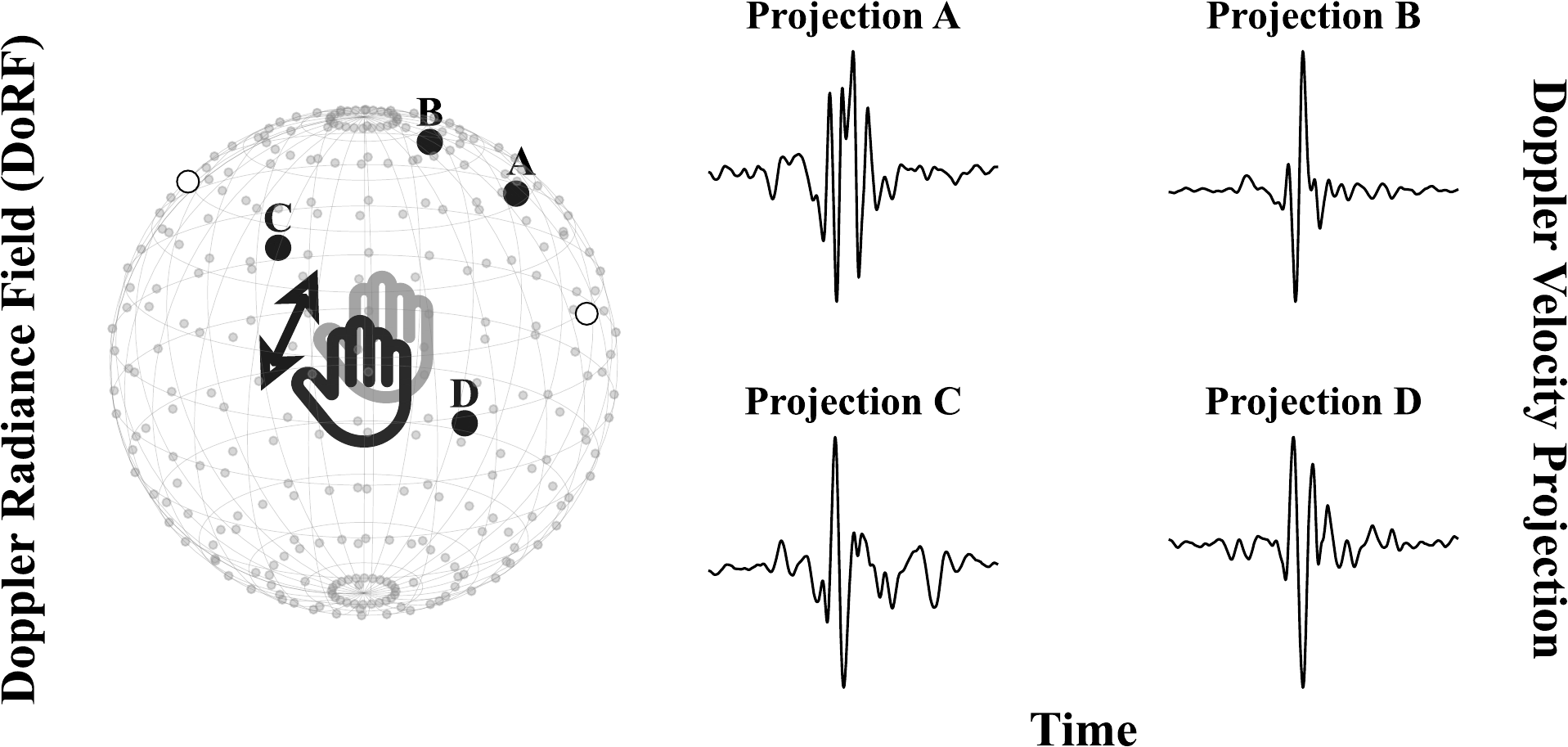}
    }
    
    \caption{Illustration of a DoRF constructed from Wi-Fi CSI. Sphere dots indicate normalized effective observation directions, each yielding a one-dimensional Doppler projection. Four examples capture the same hand motion from different viewpoints. Full-sphere equiangular re-projection provides an ordered and full-sphere directional encoding of the recovered latent motion.}
\label{figure:Dorf}
\end{figure}

The following subsections describe the intrinsic orthogonal ambiguity of the recovered factors and state an ideal unit-direction identifiability result that motivates spherical representation learning.

\subsection{Solution Equivalence and Identifiability}\label{subsec:ambiguity}

The reconstruction term in \eqref{eq:objective_re} depends on $\mathbf{V}$ and $\mathbf{R}$ only through their product, while the isotropic Frobenius regularizers are invariant under orthogonal transformations. Indeed, for any orthogonal matrix $\mathbf{Q}$ of compatible dimension, that is, any $\mathbf{Q}$ with $\mathbf{Q}^\top\mathbf{Q}=\mathbf{I}$,
\begin{equation}
\begin{gathered}
\mathbf{V}\mathbf{R}=(\mathbf{V}\mathbf{Q})\bigl(\mathbf{Q}^\top\mathbf{R}\bigr),\\
\|\mathbf{V}\mathbf{Q}\|_F=\|\mathbf{V}\|_F,\quad
\|\mathbf{Q}^\top\mathbf{R}\|_F=\|\mathbf{R}\|_F,
\end{gathered}
\label{eq:orth_equiv}
\end{equation}
since $\|\mathbf{Q}^\top\mathbf{R}\|_F^2=\mathrm{tr}(\mathbf{R}^\top\mathbf{Q}\mathbf{Q}^\top\mathbf{R})=\mathrm{tr}(\mathbf{R}^\top\mathbf{R})$ and likewise for $\mathbf{V}\mathbf{Q}$. This standard invariance of the Frobenius norm under orthogonal transformations is elementary and holds in any dimension, not only for the $3\times 3$ case arising here.
Consequently, the objective in \eqref{eq:objective_re} is unchanged when $\mathbf{V}$ and $\mathbf{R}$ are replaced by $\mathbf{V}\mathbf{Q}$ and $\mathbf{Q}^\top\mathbf{R}$. The factorization therefore retains an intrinsic orthogonal ambiguity, and equivalent recovered coordinate systems can differ by a global rotation or reflection, which arises because Wi-Fi multipath measurements provide no external 3D coordinate frame. The same ambiguity applies to the closed-form solution \eqref{eq:svd_solution}, since for a fixed optimal product the balanced factors form one orbit under a common orthogonal transformation, while the product itself is unique except in the cutoff-tie case noted above.

The following proposition characterizes an ideal unit-direction model with sufficiently diverse viewing directions, providing a geometric identifiability reference for the normalized direction model used to interpret the recovered factors.

\begin{lemma}[Ideal uniqueness of unit viewing directions up to an orthogonal transformation]
\label{lem:unique_mod_rot}
Assume two solutions with $\mathbf{R}_1,\mathbf{R}_2\in\mathbb{R}^{3\times N}$ satisfy
\[
\mathbf{V}_r=\mathbf{V}_1\mathbf{R}_1=\mathbf{V}_2\mathbf{R}_2,
\qquad
\mathrm{rank}(\mathbf{V}_r)=3,
\]
with unit-norm columns such that $\|\mathbf{r}_{1,i}\|=\|\mathbf{r}_{2,i}\|=1$ for all $i$. If the set $\{\mathbf{r}_{1,i}\mathbf{r}_{1,i}^\top\}_{i=1}^{N}$ spans the space of symmetric $3\times 3$ matrices, then there exists an orthogonal matrix $\mathbf{Q}$ such that
\[
\mathbf{V}_2=\mathbf{V}_1\mathbf{Q},
\qquad
\mathbf{R}_2=\mathbf{Q}^\top\mathbf{R}_1.
\]
\end{lemma}

\begin{proof}
Please refer to the \hyperref[Appendix:A]{Appendix}.
\end{proof}

Together with the orthogonal invariance in \eqref{eq:orth_equiv}, Proposition~\ref{lem:unique_mod_rot} shows that, under the stated ideal exact unit-direction factorization assumptions, the recovered solution is unique up to a global rotation or reflection, provided the effective viewing directions are sufficiently diverse. It is not a statement about the regularized estimator actually defined by \eqref{eq:objective_re}, whose normalized columns are used only for directional interpretation. Since the space of symmetric $3\times 3$ matrices has dimension $6$, at least six sufficiently diverse directions are needed. Six CSI-ratio streams do not by themselves guarantee the condition, because pairwise differential directions may be geometrically or algebraically dependent.

\section{DoRF++: Spherical Representation Learning over DoRF}\label{section:dorfpp}

Everything up to this point constitutes DoRF, the representation. This section introduces DoRF++, a recognition model that operates on the DoRF as a spherical representation rather than on flattened descriptors. The obtained DoRFs are time-varying signals defined on the sphere \(\mathbb{S}^2\), one per receive antenna. DoRF++ maps these receive-antenna-specific representations to an activity label while preserving their directional structure.

Let \(A_r\) denote the number of receive antennas, as in Section~\ref{section:csi}, \(T\) the number of time steps, and \(M\) the number of colatitude samples in the equiangular grid. The multi-DoRF input is represented as
\begin{equation}
\mathbf{P} \in \mathbb{R}^{A_r \times T \times M \times 2M},
\end{equation}
where \(\mathbf{P}^{(a)} \in \mathbb{R}^{T \times M \times 2M}\) denotes the DoRF associated with the \(a\)-th receive antenna. This structured representation preserves receiver-specific spherical motion patterns while limiting the influence of noisy Doppler projections from one receive antenna on the features derived from the others.

Applying conventional planar operations directly to each \(\mathbf{P}^{(a)}\) can distort the spherical geometry, because Euclidean processing does not account for longitudinal periodicity, the convergence of longitudes near the poles, or the non-uniform spherical area represented by different grid cells. We therefore employ a spherical Transformer with quadrature-aware attention. By sharing the same operations across spherical locations and aggregating information over the full sphere, this architecture reduces sensitivity to changes in the recovered global orientation, improving robustness when the effective multipath viewing geometry varies across environments. Quadrature weighting alone, however, does not guarantee exact rotation equivariance, so the resulting robustness remains approximate.

Let \(\{\mathbf{d}_j\}_{j=1}^{G}\) be the fixed equiangular grid on \(\mathbb{S}^2\) used to construct each DoRF, with \(G=2M^2\) for an \(M\times 2M\) latitude--longitude grid. For the \(a\)-th receive antenna, the DoRF time series at direction \(\mathbf{d}_j\) is denoted by
\begin{equation}
\mathbf{p}^{(a,j)}
\triangleq
\bigl[
\mathbf{P}^{(a)}(0,\mathbf{d}_j),
\dots,
\mathbf{P}^{(a)}(T{-}1,\mathbf{d}_j)
\bigr]^\top
\in \mathbb{R}^{T}.
\end{equation}
The collection \(\{\mathbf{p}^{(a,j)}:a=1,\ldots,A_r,\;j=1,\ldots,G\}\) must then be mapped to an activity label while respecting the spherical geometry of the direction index \(j\). Although the directional samples are not independent, being projections of the same recovered latent velocity sequence, their organization on the sphere allows the classifier to model local and long-range relationships among motion projections along different directions.

\subsection{Per-direction temporal feature extraction}

Each \(\mathbf{p}^{(a,j)}\) is a one-dimensional Doppler time series. We extract a fixed, high-dimensional temporal feature vector from each direction using the Random Convolutional Kernel Transform~\cite{dempster2020rocket}. Because the maximum and proportion-of-positive-value statistics are nonlinear, applying them after directional mixing yields features that differ from those obtained along the three latent axes alone. A bank of \(K_{\rm ker}\) random 1D convolution kernels is generated once using a fixed random seed and then kept fixed, shared across all spherical directions, receive antennas, and samples. Each kernel uses a randomly selected length from \(\{7,9,11\}\), weights sampled from a normal distribution, and a dilation factor chosen as a power of two, capturing temporal patterns across multiple scales. We adopt random kernels rather than a learnable CNN because they capture a rich set of multi-scale Doppler patterns in a single forward pass, at low computational cost and with reduced risk of overfitting on the relatively small HAR datasets common in wireless sensing.

Applying this kernel bank to \(\mathbf{p}^{(a,j)}\) produces multiple feature maps, from which we compute the maximum activation and one or more proportions of positive values (PPV), obtained by shifting the responses with the configured bias terms and measuring the fraction of samples above zero. The resulting feature vector is \(\mathbf{f}^{(a,j)}=(f^{(a,j)}_{1},\ldots,f^{(a,j)}_{d_{\rm feat}})\in\mathbb{R}^{d_{\rm feat}}\), where \(d_{\rm feat}\) is determined by the number of kernels and summary statistics. The same extractor is applied to all directions and all receive-specific DoRFs, so all features lie in a common temporal feature space.

\subsection{Spherical tokenization and quadrature weights}

For each receive antenna, every direction is treated as a token located at \(\mathbf{d}_j\in\mathbb{S}^2\), yielding a receiver-specific spherical feature signal \(\mathbf{F}^{(a)}:\mathbb{S}^2\rightarrow\mathbb{R}^{d_{\rm feat}}\) with \(\mathbf{F}^{(a)}(\mathbf{d}_j)=\mathbf{f}^{(a,j)}\). On an equiangular latitude--longitude grid, the sampling density is not uniform with respect to spherical area. Hence, a positive quadrature weight \(\omega_j\) is associated with each token according to the area represented by that grid cell. Using midpoint quadrature and letting \(\theta_j\in(0,\pi)\) denote the colatitude of \(\mathbf{d}_j\), the normalized weights are
\begin{equation}
\omega_j
=
\frac{\sin\theta_j}
{\sum_{j'=1}^{G}\sin\theta_{j'}}.
\label{eq:spherical_quadrature_weights}
\end{equation}
These weights allow the discrete attention operation to approximate continuous integration with respect to the spherical area measure.

Before applying attention, the feature vector at each token is projected to the model dimension \(d_{\rm model}\) as \(\mathbf{z}_{a,j}^{(0)}=\mathbf{W}_{\rm in}\,\mathbf{f}^{(a,j)}+\mathbf{b}_{\rm in}\in\mathbb{R}^{d_{\rm model}}\), \(j=1,\dots,G\), where \(\mathbf{W}_{\rm in}\in\mathbb{R}^{d_{\rm model}\times d_{\rm feat}}\) is a learnable linear projection and \(\mathbf{b}_{\rm in}\) a bias vector. Unlike standard text positional encodings, which assume a one-dimensional sequence order, these tokens represent directions on \(\mathbb{S}^2\). Rather than assigning an absolute positional encoding to each direction, spherical position is represented through a relative positional bias between pairs of directions during attention. This gives the attention layers a smooth encoding of the relative geometry between directions, allowing them to use the spherical organization rather than treating the tokens as unordered, while avoiding dependence on a particular absolute coordinate frame.

\subsection{Quadrature-aware spherical self-attention}

We apply spherical attention layers independently to each receive-specific DoRF while sharing the attention parameters across receive antennas. For a single head and receive antenna \(a\), the queries, keys, and values are computed by linear projections:
\begin{equation}
\mathbf{q}_{a,j}
=
\mathbf{W}_Q\mathbf{z}^{(\ell)}_{a,j},
\qquad
\mathbf{k}_{a,j}
=
\mathbf{W}_K\mathbf{z}^{(\ell)}_{a,j},
\qquad
\boldsymbol{\nu}_{a,j}
=
\mathbf{W}_V\mathbf{z}^{(\ell)}_{a,j},
\end{equation}
where \(\mathbf{W}_Q,\mathbf{W}_K,\mathbf{W}_V\in\mathbb{R}^{d_h\times d_{\rm model}}\) are learnable weight matrices and \(d_h\) is the head dimension. Layer and head indices are omitted for readability.

Quadrature-aware spherical attention extends conventional Transformers to signals defined on \(\mathbb{S}^2\). Queries, keys, and values are still obtained through linear projections, but attention is computed over spherical directions, and the attention score incorporates the relative spherical geometry between the corresponding directions in addition to the content-based similarity. For directions \(\mathbf{d}_j\) and \(\mathbf{d}_{j'}\), we first define their relative angular similarity as \(\xi_{jj'}=\mathbf{d}_j^\top\mathbf{d}_{j'}\) and represent it using Legendre polynomial features
\begin{equation}
\boldsymbol{\psi}_{jj'}
=
\left[
P_0(\xi_{jj'}),
P_1(\xi_{jj'}),
\ldots,
P_{L_{\rm pos}}(\xi_{jj'})
\right]^\top,
\end{equation}
where \(P_l(\cdot)\) denotes the Legendre polynomial of degree \(l\) and \(L_{\rm pos}\) controls the maximum degree used. A learnable scalar relative positional bias is then obtained as \(b_{jj'}^{\rm rel}=\mathrm{MLP}_{\rm rel}(\boldsymbol{\psi}_{jj'})\). Since \(\xi_{jj'}\) depends only on the inner product between two directions, it is unchanged under any common orthogonal transformation \(\mathbf{Q}\in O(3)\), because \((\mathbf{Q}\mathbf{d}_j)^\top(\mathbf{Q}\mathbf{d}_{j'})=\mathbf{d}_j^\top\mathbf{d}_{j'}\). The relative positional bias is therefore invariant to a global rotation or reflection of the spherical coordinate system, directly reducing sensitivity to the orthogonal ambiguity of the recovered DoRF.

Each key--value contribution is additionally weighted by the quadrature weight \(\omega_j\) of its direction, making the discrete attention operation a quadrature approximation of a continuous integral over the sphere. Attention thus accounts for the non-uniform spherical area represented by different grid points while enabling interactions among directional tokens over the full sphere. The attention output at token \(j\) is
\begin{align}
\widetilde{\mathbf{z}}^{(\ell)}_{a,j}
&=
\sum_{j'=1}^{G}
a^{(a)}_{jj'}\,\boldsymbol{\nu}_{a,j'},
\\
a^{(a)}_{jj'}
&=
\frac{
\exp\!\Bigl(
\mathbf{q}_{a,j}^\top\mathbf{k}_{a,j'}/\sqrt{d_h}
+
b_{jj'}^{\rm rel}
\Bigr)\,\omega_{j'}
}{
\sum_{j''=1}^{G}
\exp\!\Bigl(
\mathbf{q}_{a,j}^\top\mathbf{k}_{a,j''}/\sqrt{d_h}
+
b_{jj''}^{\rm rel}
\Bigr)\,\omega_{j''}
}.
\label{eq:s2_attention_discrete}
\end{align}
Here, \(\widetilde{\mathbf{z}}^{(\ell)}_{a,j}\in\mathbb{R}^{d_h}\), and the quadrature weights ensure that the attention operation approximates integration with respect to the spherical measure rather than treating all sampled directions as equally representative. The finite spherical discretization, however, means that the resulting discrete operation is not guaranteed to be exactly rotation-equivariant.

Multi-head attention concatenates the outputs of several heads and projects the result back to dimension \(d_{\rm model}\). Each spherical Transformer block follows a pre-norm residual design:
\begin{align}
\mathbf{u}^{(\ell)}_{a,j}
&=
\mathbf{z}^{(\ell)}_{a,j}
+
\left[
\mathrm{MHA}_{\mathbb{S}^2}
\!\left(
\{\mathrm{Norm}(\mathbf{z}^{(\ell)}_{a,j'})\}_{j'=1}^{G}
\right)
\right]_{j},
\\
\mathbf{z}^{(\ell+1)}_{a,j}
&=
\mathbf{u}^{(\ell)}_{a,j}
+
\mathrm{MLP}
\!\left(
\mathrm{Norm}(\mathbf{u}^{(\ell)}_{a,j})
\right),
\\
&\ell=0,\dots,L_{\mathrm{attn}}{-}1,\notag
\end{align}
where \(\mathrm{MHA}_{\mathbb{S}^2}\) is the global quadrature-aware spherical multi-head attention defined in~\eqref{eq:s2_attention_discrete}, and \(L_{\mathrm{attn}}\) is the number of spherical attention blocks. The same spherical Transformer is applied to each receive-specific DoRF, letting the model learn a shared spherical representation while preserving receiver-specific information.

\begin{figure*}[!t]
\centering
\subfloat{%
\includegraphics[width=0.99\linewidth]{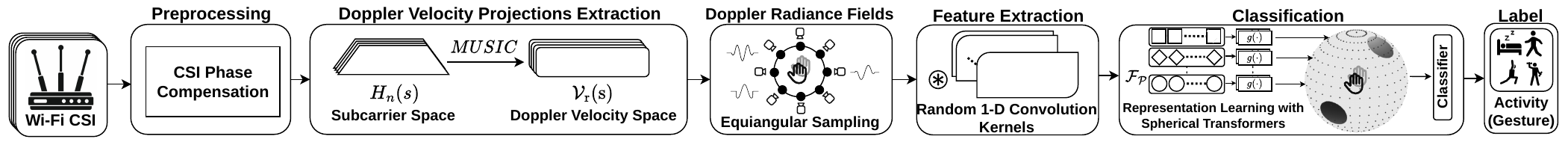}
}
\caption{The proposed DoRF++ framework for Wi-Fi-based HAR. CSI signals are first transformed into Doppler velocity projections that capture motion through multiple sparse effective directions. These projections are then used to recover one latent motion representation for each receive antenna, which is re-projected onto a fixed equiangular spherical grid to construct a receiver-specific DoRF. A shared temporal feature extractor, spherical Transformer blocks, receiver-level fusion, and a classifier are subsequently used to recognize the performed activity.}
\label{figure:method}
\end{figure*}

\subsection{Pooling, multi-DoRF fusion, and classification}

After \(L_{\mathrm{attn}}\) spherical attention blocks, each receive-specific DoRF is aggregated into a sphere-level descriptor. For receive antenna \(a\), we compute
\begin{equation}
\mathbf{g}^{(a)}
=
\operatorname{Pool}_{\mathbb{S}^2}
\left(
\{\mathbf{z}^{(L_{\mathrm{attn}})}_{a,j}\}_{j=1}^{G}
\right)
\in\mathbb{R}^{d_{\rm model}},
\end{equation}
where \(\operatorname{Pool}_{\mathbb{S}^2}\) summarizes the learned directional responses over the sphere. A max-type spherical pooling operation is used, retaining the most discriminative directional responses while reducing sensitivity to their exact recovered locations. This does not guarantee rotation or reflection invariance, but it reduces the dependence of the receiver-specific descriptor on any single spherical token.

The receiver-specific descriptors are then fused into a single activity representation using an element-wise max operation over the receive-antenna dimension:
\begin{equation}
\mathbf{g}
=
\operatorname*{MaxPool}_{a=1,\ldots,A_r}
\left\{
\mathbf{g}^{(a)}
\right\}
\in\mathbb{R}^{d_{\rm model}},
\qquad
g_k
=
\max_{a=1,\ldots,A_r}
g_k^{(a)}.
\label{eq:antenna_max_pooling}
\end{equation}
Since the same spherical Transformer and pooling operation are applied to all receive-specific DoRFs, this symmetric fusion is invariant to permutations of the receive-antenna order and retains the strongest learned response for each feature dimension across antennas, so that motion information observed under different propagation conditions is combined independently of antenna indexing. Fusion therefore discards which antenna produced each maximum, but each receive-specific field contributes to the pooled descriptor through the preceding attention blocks, which process it independently.

Finally, a classifier head \(h(\cdot)\) maps \(\mathbf{g}\) to class logits, and softmax yields class probabilities \(\widehat{\mathbf{y}}=\operatorname{softmax}(h(\mathbf{g}))\in\mathbb{R}^{C}\), where \(C\) is the number of activity classes. This stage preserves the directional organization of each DoRF while enabling long-range interactions over the sphere and combining complementary receiver-specific representations through permutation-invariant antenna pooling. DoRF++ therefore provides a spherical-geometry-aware and antenna-order-invariant activity embedding suited to generalization across users and sensing conditions. Fig.~\ref{figure:method} summarizes the complete pipeline.

\section{Experiments}\label{section:experiments}

\subsection{Data}

\begin{figure}[!b]
    \centering
    \subfloat{%
        \includegraphics[width=0.8\linewidth]{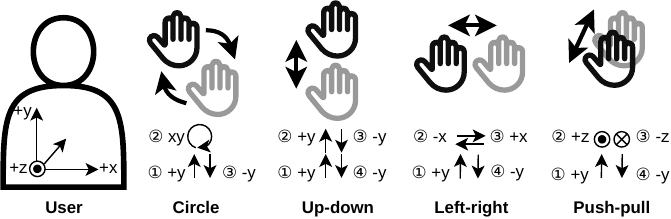}
    }
    \caption{Illustration of the four right-hand gestures included in the
    UTHAMO-5G dataset: \textit{circle}, \textit{up--down},
    \textit{left--right}, and \textit{push--pull}.}
    \label{figure:gestures}
\end{figure}

To evaluate the proposed method, we collected a challenging hand-motion
dataset, UTHAMO-5G. Unlike conventional Wi-Fi-based HAR datasets that
focus on large-scale activities such as walking, sitting, or standing,
UTHAMO-5G is designed around small-scale right-hand gestures that
produce subtle wireless-channel variations. It includes CSI recordings
of four gestures, \textit{circle}, \textit{left--right},
\textit{up--down}, and \textit{push--pull}, performed by ten adult
participants in a fixed indoor office environment of size
\(11\,\mathrm{m}\times5.6\,\mathrm{m}\). All participants provided
informed consent. The gestures are illustrated in
Fig.~\ref{figure:gestures}.

The gestures were intentionally chosen to be challenging for
Wi-Fi-based sensing, since several pairs exhibit similarity under
rotations or permutations of the spatial \(x\), \(y\), and \(z\) axes.
For example, a \textit{push--pull} motion in the sagittal direction may
produce Doppler projections resembling those of \textit{left--right} or
\textit{up--down} motions performed in other planes. This design makes
the dataset particularly suitable for evaluating whether Wi-Fi-based
HAR methods can distinguish fine-grained hand motions whose projections
may be ambiguous from certain viewpoints.

Data were collected using a single body orientation, with each
participant facing the transmitter, as shown in Fig.~\ref{figure:map}.
The setup consisted of four ASUS RT-AC86U Wi-Fi access points, each
equipped with three external antennas and one internal antenna. One
access point was configured as the transmitter, while the remaining
three operated as passive sniffers serving as receivers. The
transmitter operated in the
\(5\,\mathrm{GHz}\) band with an \(80\,\mathrm{MHz}\) bandwidth on
channel 157 and transmitted packets at an average rate of
\(147\) packets per second, yielding CSI measurements with
\(256\) subcarriers for each recorded antenna stream.

\begin{figure}[!b]
        \centering		\includegraphics[width=0.85\linewidth]{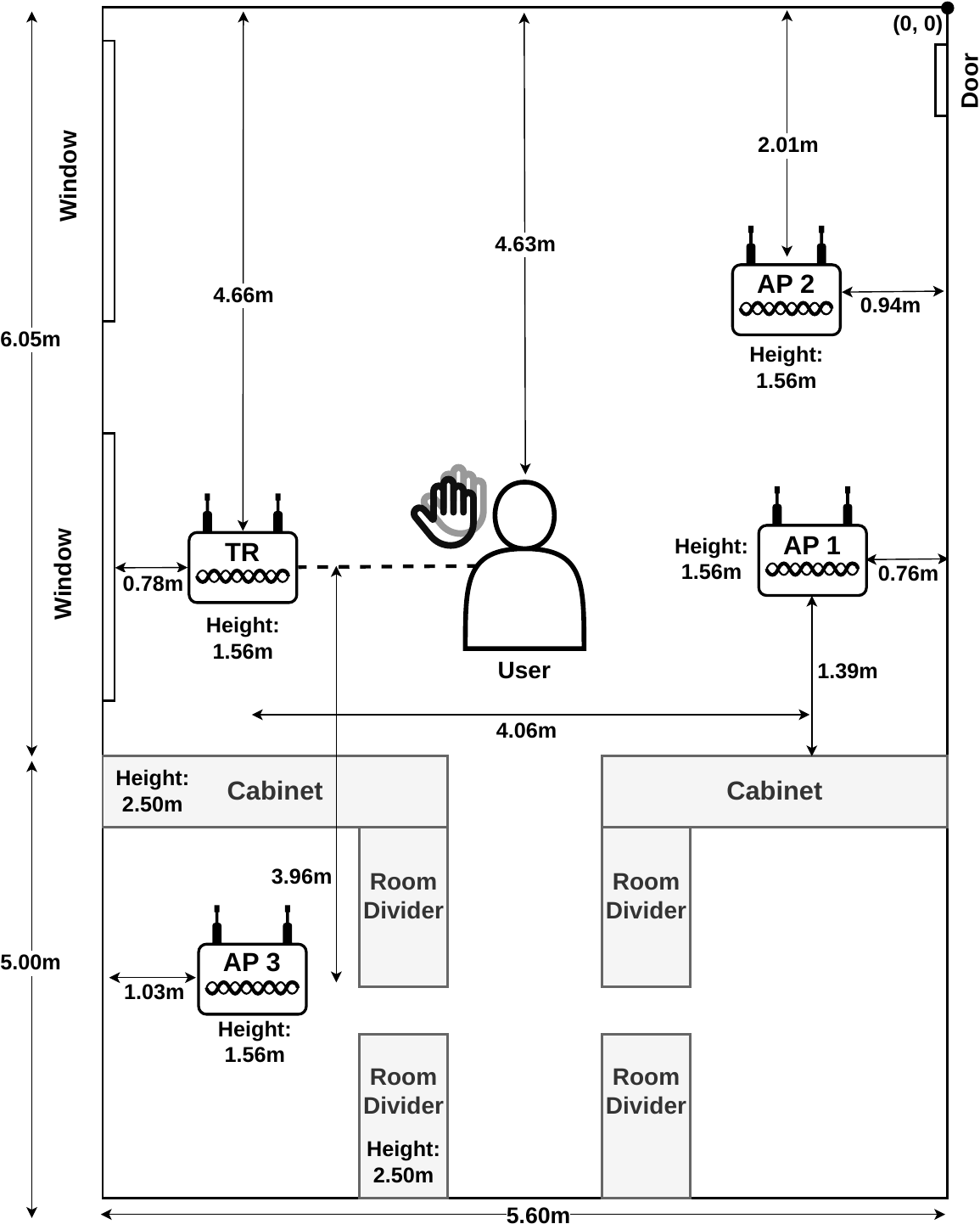}
	
	\caption{The detailed floor plan of the UTHAMO-5G data collection setup.}
	\vspace{-0mm}
\label{figure:map}
\end{figure}

The three receiver access points were positioned to provide different
propagation and motion-observation conditions. AP1 was placed such that
the participant was near the middle of the line-of-sight path between
AP1 and the transmitter, allowing the hand motion to directly perturb
the dominant LOS component. AP2 had unobstructed paths toward both the
transmitter LOS region and the participant, whereas AP3 was placed
behind cabinets and had a direct path toward neither. These placements
introduce diverse multipath conditions and different levels of motion
visibility, with all locations and heights shown in
Fig.~\ref{figure:map}.

CSI was collected using the UbiLocate framework. We also developed
custom data-collection software to synchronize packets received by the
different access points and to guide participants in performing the
instructed gesture at the appropriate time. Each
participant completed two separate recording sessions, performing
\(10\) trials of each of the four gestures per session, resulting in
\(80\) trials per participant and \(800\) trials in total. The dataset
is class-balanced, with \(200\) trials for each gesture.

For each trial, participants performed the instructed gesture within a
six-second activity window. To reduce boundary effects at gesture
initiation and termination, the first and last \(0.5\) seconds were
removed, and the remaining five-second segment was resampled in the
Doppler-projection domain onto a uniform \(100\,\mathrm{Hz}\) grid,
yielding \(500\) temporal samples per trial.

\begin{table*}[!t]
\centering
\caption{Four-class cross-user hand-motion generalization accuracy (\%). Results are reported as mean $\pm$ standard deviation.}
\label{table:results}
\begin{adjustbox}{width=0.98\textwidth}
\begin{tabular}{lll|ccc|c}
\toprule
\textbf{Method} & \textbf{Input Signal} & \textbf{Classifier}
  & \rev{\textbf{Single Receiver AP $1$}} 
  & \rev{\textbf{Single Receiver AP $2$}} 
  & \rev{\textbf{Single Receiver AP $3$}} 
  & \textbf{Average} \\
\midrule
AMAP~\cite{salehinejad2023joint} & Raw CSI Magnitude & Logistic Regression     
  & 31.8\%\,$\pm$\,3.1\%  
  & 28.2\%\,$\pm$\,2.8\%  
  & 26.7\%\,$\pm$\,2.6\%  
  & 28.9\%\,$\pm$\,2.8\% \\
\midrule
CMAP~\cite{salehinejad2023joint} & Raw CSI Magnitude & Logistic Regression        
  & 30.6\%\,$\pm$\,3.0\%  
  & 27.3\%\,$\pm$\,2.7\%  
  & 25.5\%\,$\pm$\,2.5\%  
  & 27.8\%\,$\pm$\,2.7\% \\
\midrule
CapsHAR~\cite{djogoHAR} & CSI Magnitude & Capsule Neural Network
  & 33.2\%\,$\pm$\,4.3\%  
  & 29.6\%\,$\pm$\,3.7\%  
  & 27.5\%\,$\pm$\,3.4\%  
  & 30.1\%\,$\pm$\,3.8\% \\
\midrule
CSI Ratio Model~\cite{wu2022wifi} & CSI Ratio Phase & Logistic Regression      
  & 35.8\%\,$\pm$\,4.5\%  
  & 31.9\%\,$\pm$\,4.1\%  
  & 29.5\%\,$\pm$\,3.8\%  
  & 32.4\%\,$\pm$\,4.1\% \\
\midrule
APNSS + APSC~\cite{102859317} & Doppler Velocity & Logistic Regression
  & 44.2\%\,$\pm$\,5.9\%  
  & 39.1\%\,$\pm$\,5.3\%  
  & 36.1\%\,$\pm$\,4.8\%  
  & 39.8\%\,$\pm$\,5.3\% \\
\midrule
MORIC~\cite{hasanzadeh2025moric} & \rev{Delay-Separated Doppler Velocities} & MORIC Classifier      
  & 66.4\%\,$\pm$\,7.8\%  
  & 55.9\%\,$\pm$\,7.1\% 
  & 42.4\%\,$\pm$\,6.8\%
  & 54.9\%\,$\pm$\,7.2\% \\
\midrule
DoRF~\cite{hasanzadeh2025dorf} & Doppler Radiance Field & MORIC Classifier      
  & 76.5\%\,$\pm$\,8.6\%  
  & 60.2\%\,$\pm$\,7.2\% 
  & 46.9\%\,$\pm$\,10.7\%
  & 61.2\%\,$\pm$\,8.8\% \\
\midrule
Ours (DoRF++) & DoRF & Spherical Attention Networks
  & \textbf{80.4\%\,$\pm$\,8.4\%}  
  & \textbf{66.7\%\,$\pm$\,6.0\%} 
  & \textbf{54.2\%\,$\pm$\,8.6\%}
  & \textbf{67.1\%\,$\pm$\,7.6\%} \\
\bottomrule
\end{tabular}
\end{adjustbox}
\end{table*}

\subsection{Training and Test Procedure}

DoRF++ was implemented in PyTorch using the
\textit{torch-harmonics} library\footnote{\url{https://github.com/NVIDIA/torch-harmonics}},
with \textit{AttentionS2}~\cite{bonev2025attention} used as the quadrature-aware spherical
attention module. The model contains \(L_{\mathrm{attn}}=4\) spherical attention
blocks, each with \(N_h=2\) attention heads, and uses a model dimension
of \(d_{\rm model}=256\). The relative spherical positional bias uses
Legendre polynomial features up to degree \(L_{\rm pos}=5\), followed
by a relative-position MLP with hidden dimension \(128\) and \(ReLU\)
activation, learned separately for each attention head in each block.
The DoRF is constructed on an \(M\times2M\) equiangular grid with
\(M=6\).

For Doppler-projection extraction, the common-RX CSI-ratio streams are
processed using MUSIC with the number of dominant Doppler components
set to one. The pseudo-spectrum is evaluated over \(-32~\mathrm{Hz}\)
to \(32~\mathrm{Hz}\) using a uniform grid of \(513\) points, a
resolution of \(0.125~\mathrm{Hz}\). Each stream is divided into
overlapping temporal windows of \(W=32\) samples with a stride of one
sample, and the dominant Doppler frequency in each window is converted
to a velocity projection using the carrier wavelength. For DoRF
construction, \(\mu=0.1\), \(\gamma=0.01\), \(\epsilon_{\mathrm{tol}}=10^{-6}\), and
the maximum number of alternating-optimization iterations is \(10\).
The reported results were obtained with this alternating implementation, except for the SVD configuration in Table~\ref{table:ablation_dorfpp}.

For per-direction temporal feature extraction, a bank of
\(K_{\rm ker}=1{,}000\) random convolutional kernels was used through
the \textit{sktime} library, each with a length selected from
\(\{7,9,11\}\), weights sampled from a normal distribution, and a
dilation selected as a power of two. The resulting feature vectors were
projected into the \(256\)-dimensional model space before the spherical
attention blocks. The model was trained using cross-entropy loss with
label smoothing of \(0.1\) and the Adam
optimizer~\cite{kingma2014adam}, with a learning rate of
\(1\times10^{-6}\), a batch size of \(64\), and at most \(2500\)
epochs. Early stopping with a patience of \(200\) epochs was applied,
and the checkpoint with the lowest validation loss was retained for
evaluation.

Both the transmitter and each receiver AP are equipped with four
antennas, yielding \(4\times4=16\) raw TX--RX CSI streams per receiver
AP. For each receive antenna, we form CSI ratios between all unordered
pairs of transmit antennas sharing that receive antenna. Since four
transmit antennas produce \(\binom{4}{2}=6\) unordered pairs, each
receive antenna yields six CSI-ratio streams and six Doppler
projections.

A separate DoRF is constructed for each of the four receive antennas.
Thus, for a given receiver AP, the \(24\) ratio-derived projections are
organized into four receive-antenna-specific groups, with each DoRF
recovered independently from the six projections of one receive
antenna. Although the pairwise ratio projections are not all
statistically or algebraically independent, they provide complementary
observations of the underlying motion. This construction preserves the
common observation geometry within each group and prevents noisy
projections from one receive antenna from affecting the DoRF recovery
of the others. The four resulting DoRFs are processed using shared
spherical Transformer parameters and fused through
permutation-invariant max pooling over the receive-antenna dimension.

To evaluate cross-user generalization, we used a ten-fold
leave-one-subject-out (LOSO) cross-validation protocol. In each fold,
all data from one participant were held out exclusively for testing.
Among the remaining nine, one was used for validation and eight for
training. The training, validation, and test subsets were therefore
subject-disjoint, and no samples from the held-out test participant
were used for model selection, early stopping, or hyperparameter
tuning. This procedure was
repeated until each participant had served once as the test subject, and
final performance was obtained by aggregating the test results across the
ten folds.

\subsection{Results}

\subsubsection{Model Generalization to Unseen Users}

Table~\ref{table:results} compares the proposed DoRF++ method with existing Wi-Fi-based HAR approaches on the UTHAMO-5G dataset. The evaluation is conducted in the \rev{single receiver AP setting}, where each receiver AP is used independently. Since the three APs observe the hand motion under different propagation conditions, their results also reflect the effect of AP placement and motion visibility on cross-user generalization.

\rev{For four-class classification, DoRF++ achieves the highest accuracy among all evaluated methods. Using AP1, DoRF++ and DoRF reach $80.4\%$ and $76.5\%$, compared with $66.4\%$ for MORIC. Averaged across the three single receiver AP settings, DoRF++ reaches $67.1\%$, outperforming DoRF at $61.2\%$ and MORIC at $54.9\%$. Explicitly organizing Doppler projections into a radiance-field representation and processing them with spherical attention therefore improves generalization to unseen users.} Because the DoRF and DoRF++ rows use the same Doppler projections and differ only in how the spherical field is processed, the difference between them measures the contribution of spherical representation learning. Table~\ref{table:avg_confusion} reports the average confusion matrix for DoRF++ using AP1, where class-wise accuracy ranges from $72.49\%$ for \textit{push--pull} to $90.53\%$ for \textit{circle}. The strongest performance is obtained for \textit{circle}, which produces more distinctive Doppler variations with respect to the AP1--transmitter geometry because it involves motion along both horizontal and vertical axes, whereas \textit{left--right} and \textit{push--pull} are more frequently confused because their projections can become similar from certain propagation viewpoints.

Traditional methods that directly use the magnitude or phase of raw CSI perform poorly in the generalization setting. AMAP and CMAP~\cite{salehinejad2023joint} achieve average accuracies of only $28.9\%$ and $27.8\%$, indicating that although raw CSI magnitude contains activity-related variations, it remains highly sensitive to user-dependent effects, environmental noise, and AP-specific propagation changes. CapsHAR~\cite{djogoHAR} slightly improves the average accuracy to $30.1\%$ with a more advanced architecture, but remains close to chance level for four-class recognition. The CSI ratio model~\cite{wu2022wifi}, which mitigates phase distortions such as STO and SFO, reaches $32.4\%$, indicating that CSI phase carries motion information but is not sufficient for reliable cross-user recognition in this fine-grained setting. Doppler-velocity methods generalize better, and APNSS$\,+\,$APSC~\cite{102859317}, which selects informative antenna pairs and extracts Doppler velocity, achieves $39.8\%$.

\begin{table}[!t]
\centering
\caption{Average confusion matrix (\%) across ten subjects for DoRF++ using AP1.}
\begin{tabular}{c|cccc}
\hline
 & \textbf{Circle} & \textbf{Left--right} & \textbf{Up--down} & \textbf{Push--pull} \\
\hline
\textbf{Circle}      & \textbf{90.53} & 6.84 & 1.58 & 1.05 \\
\textbf{Left--right} & 1.05 & \textbf{76.32} & 7.89 & 14.74 \\
\textbf{Up--down}    & 3.68 & 8.30 & \textbf{82.26} & 5.76 \\
\textbf{Push--pull}  & 0.53 & 17.51 & 9.47 & \textbf{72.49} \\
\hline
\end{tabular}
\label{table:avg_confusion}
\end{table}

\begin{figure}[!b]
	\centering
	\subfloat{%
		\includegraphics[width=0.80\linewidth]{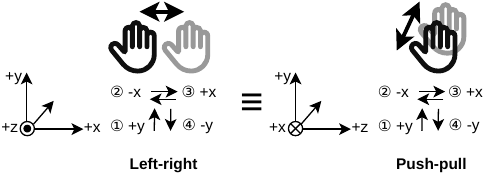}
	}
	\caption{Left--right and push--pull gestures can become equivalent under a rotation of the coordinate system, making them difficult to distinguish in Wi-Fi sensing.}
\label{figure:equivalence}
\end{figure}

\subsubsection{Impact of Access Point Location on Performance}

The results in Table~\ref{table:results} show that AP1 consistently provides the best recognition performance, followed by AP2 and AP3, consistent with the data-collection geometry in Fig.~\ref{figure:map}. AP1 was positioned such that the participant was near the middle of the line-of-sight path between AP1 and the transmitter, so the hand motion directly perturbs the dominant propagation path, producing stronger Doppler variations and more discriminative CSI changes. This explains the highest \rev{single receiver AP} accuracy of $80.4\%$. AP2 also had a clear view of both the transmitter line-of-sight path and the participant, but the participant was less directly centered on the transmitter--receiver path, so the motion-induced changes are still visible but less dominant, leading to a lower accuracy of $66.7\%$. In contrast, AP3 was placed behind cabinets and mainly captures indirect and attenuated multipath reflections. Although these NLOS components still contain motion information, they are weaker, noisier, and more sensitive to environmental scattering, resulting in the lowest accuracy of $54.2\%$.

These results indicate that the geometric relationship among the transmitter, receiver AP, participant, and surrounding reflectors still has a major impact on Wi-Fi-based HAR. When the motion directly affects a strong propagation path, Doppler projections become more informative and easier to organize into a stable DoRF representation, whereas when the AP observes the motion only through weaker multipath components, the projections become more ambiguous and generalization degrades.

\subsubsection{Two-Class Gesture Recognition}

\begin{table}[!t]
\centering
\caption{DoRF++ binary hand-motion recognition generalization accuracy using AP1.}
\resizebox{\columnwidth}{!}{%
\begin{tabular}{c|cccc}
  \hline
  & \textbf{Circle} & \textbf{Left--right} & \textbf{Up--down} & \textbf{Push--pull} \\
  \hline
  \textbf{Circle}      & \cellcolor{gray!50} & $91.9\%\pm10.1\%$ & $92.3\%\pm7.9\%$ & $95.9\%\pm6.7\%$ \\
  \textbf{Left--right}  & \cellcolor{gray!50} & \cellcolor{gray!50} & $86.5\%\pm8.1
  \%$ & $75.2\%\pm11.4\%$ \\
  \textbf{Up--down}     & \cellcolor{gray!50} & \cellcolor{gray!50} & \cellcolor{gray!50} & $87.7\%\pm11.6\%$ \\
  \textbf{Push--pull}   & \cellcolor{gray!50} & \cellcolor{gray!50} & \cellcolor{gray!50} & \cellcolor{gray!50} \\
  \hline
\end{tabular}
}
\label{table:confusion_matrix}
\vspace{-4mm}
\end{table}

The four gestures in UTHAMO-5G were deliberately selected as fine-grained hand motions with geometric similarities across the spatial $x$, $y$, and $z$ axes. Since a receiver observes only Doppler projections of the underlying spatial motion, rather than the full 3-D trajectory, two gestures that are distinct in the physical coordinate system may appear similar from a particular wireless propagation viewpoint. To analyze this behavior, a binary hand-motion classification experiment is conducted using DoRF++ with AP1. Table~\ref{table:confusion_matrix} shows that DoRF++ distinguishes most gesture pairs with high generalization accuracy. \rev{Pairs involving the \textit{circle} gesture are generally easier to distinguish, likely because circular motion contains substantial components along both the \(x\)- and \(y\)-axes, whereas \textit{up--down} motion is primarily confined to the \(y\)-axis and \textit{left--right} and \textit{push--pull} involve coupled components along the \(x\)--\(y\) and \(y\)--\(z\) axes, respectively. These partially shared directional components may make those pairs harder to separate.}

In contrast, the pair \textit{left--right vs. push--pull} remains the most challenging, with an accuracy of \rev{$75.2\%\pm11.4\%$}. As shown in Fig.~\ref{figure:equivalence}, both gestures involve back-and-forth motion along a single axis, and their Doppler signatures can become similar when projected onto certain multipath directions. For example, a \textit{push--pull} motion along the sagittal direction can resemble a \textit{left--right} motion if the effective propagation path observes the hand from a rotated or oblique viewpoint. This is consistent with the off-diagonal confusion between these classes in Table~\ref{table:avg_confusion}, where \textit{circle} likewise shows moderate confusion with \textit{left--right}, as participants may trace horizontally elongated ellipses rather than perfectly symmetric circles.

\subsubsection{Ablation Study}

Table~\ref{table:ablation_dorfpp} presents an ablation study evaluating the contribution of the main components of the DoRF++ pipeline. The complete model achieves the highest AP1 cross-user generalization accuracy of $80.4\%\pm8.4\%$, demonstrating the effectiveness of combining the DoRF representation with spherical positional encoding and spherical attention.

\begin{table}[!b]
\centering
\caption{Ablation study of DoRF++ generalization accuracy using AP1.}
\def\arraystretch{1.0}%
\begin{tabular}{lcc}
\hline
\textbf{Configuration} & \textbf{Accuracy (\%)} & \textbf{SD (\%)} \\
\hline
w/o CSI Phase Compensation & 31.2\% & 3.6\% \\
Conventional same-TX/diff.-RX CSI ratio & 66.9\% & 8.9\% \\

w/o Spherical Position Encoding & 74.1\% & 6.7\% \\

\revised{Doppler End-to-end LSTM w/o DoRF}  &  \revised{63.2\%} & \revised{5.9\%} \\
\revised{Doppler End-to-end CNN w/o DoRF}  &  \revised{69.4\%} & \revised{7.6\%} \\
PCA (3 Components) + MORIC Classifier & 60.9\% & 4.9\% \\
PCA (3 Components) + MLP Classifier & 62.1\% & 6.1\% \\
Direct Latent $\mathbf{V}$ Representation & 68.2\% & 7.6\% \\
Simple MLP Classifier on DoRF & 73.4\% & 7.9\% \\
Random 3D rotations of DoRF at test time & 80.3\% & 8.5\% \\
SVD in place of alternating estimation & 80.2\% & 8.1\% \\
\hline
\textbf{Full DoRF++} & \textbf{80.4\%} & \textbf{8.4\%} \\
\hline
\end{tabular}
\label{table:ablation_dorfpp}
\end{table}

Removing CSI phase compensation causes the largest degradation, reducing the accuracy to $31.2\%\pm3.6\%$ and confirming that phase compensation is essential for reliable Doppler extraction, since uncorrected STO, SFO, and hardware-induced distortions obscure the subtle phase variations associated with hand motion. Using the conventional same-TX/different-RX CSI-ratio construction instead of the proposed common-RX ratio, with all other settings unchanged, achieves $66.9\%\pm8.9\%$, indicating that forming ratios between streams sharing the same receive antenna provides more reliable phase sanitization. Removing spherical positional encoding decreases the accuracy to $74.1\%\pm6.7\%$, indicating that explicitly encoding the spatial arrangement of the projections on the sphere provides useful geometric information. Applying a simple MLP classifier to random-convolutional features extracted from the DoRF achieves $73.4\%\pm7.9\%$. Although this shows that DoRF contains discriminative motion information, the gap to the complete model highlights the benefit of spherical attention for capturing relationships among Doppler projections distributed across the sphere.

To assess whether the DoRF representation itself provides an advantage over direct sequence modeling, we compare it with direct Doppler-sequence LSTM and CNN baselines. The LSTM classifier processes the Doppler projections as multivariate time series using a two-layer LSTM with a hidden size of \(256\), achieving $63.2\%\pm5.9\%$. The CNN classifier represents them in a two-dimensional projection--time format and applies three convolutional blocks with \(3\times3\) kernels and channel dimensions \(1\!\rightarrow\!32\!\rightarrow\!64\!\rightarrow\!128\), achieving $69.4\%\pm7.6\%$. We further replace the DoRF recovery stage with PCA, retaining the first three principal components and processing them with the same random convolutional feature extraction scheme. The MORIC classifier then achieves $60.9\%\pm4.9\%$ and an MLP with comparable hyperparameters achieves $62.1\%\pm6.1\%$. Ordinary PCA scores need not coincide with the balanced, regularized descriptor of \eqref{eq:svd_solution}, since centering and singular-value scaling may differ, and the spherical re-projection and the subsequent nonlinear temporal extraction introduce further differences. The gap therefore supports the tested DoRF-based pipeline over this PCA baseline. It does not establish that the factorization in \eqref{eq:objective_re} imposes additional physical geometric constraints.

To further isolate the contribution of the spherical representation, we use the recovered latent velocity sequence \(\mathbf{V}\) directly with the MORIC classifier, which achieves $68.2\%\pm7.6\%$. Its lower accuracy indicates that re-projecting the recovered latent motion onto the sphere yields a more effective structured representation than the latent velocity sequence alone. Together with the direct Doppler-sequence and PCA baselines, these results suggest that the structured latent-motion recovery of DoRF offers an advantage over both direct sequence modeling and generic dimensionality reduction, while the further improvement from $73.4\%$ to $80.4\%$ indicates that spherical positional encoding and attention let DoRF++ exploit the geometric relationships among directional motion projections more effectively.

To evaluate robustness to coordinate-system orientation, we draw one rotation $\mathbf{Q}$ per test trial, sampled uniformly from the Haar measure on $SO(3)$. We apply this rotation to the latent representation, $\mathbf{V}\mapsto\mathbf{V}\mathbf{Q}$, and recompute its DoRF on the original fixed spherical grid as
\[
\mathbf{P}_{\mathbf{Q}}=(\mathbf{V}\mathbf{Q})\mathbf{D}_{\rm sph},
\]
where the columns of $\mathbf{D}_{\rm sph}$ are the fixed spherical grid directions and $\mathbf{P}_{\mathbf{Q}}$ collects the rotated field over time and grid directions. The training data and trained model remain unchanged. The accuracy is $80.3\%\pm8.5\%$, compared with the original result of $80.4\%\pm8.4\%$. This nearly unchanged performance supports the coordinate-frame independence of spherical representation learning and its robustness to arbitrary coordinate-system orientations, although exact rotation invariance on the finite grid is not guaranteed.

Replacing the alternating estimation algorithm with the closed-form SVD solution in \eqref{eq:svd_solution} yields an accuracy of $80.2\%\pm8.1\%$, compared with $80.4\%\pm8.4\%$ for alternating estimation. The mean accuracies differ by only $0.2$ percentage points, indicating that both recovery methods perform very similarly and yield comparable recognition results in this setting.

\section{Conclusion}\label{section:conclusion}

This work introduces DoRF++, a spherical Doppler representation and a representation-learning model for Wi-Fi-based HAR. DoRF represents human motion in a latent three-dimensional velocity space by interpreting Doppler velocity projections extracted from CSI as sparse observations of an underlying 3-D motion sequence. Inspired by NeRF, which reconstructs coherent 3-D structure from multiple views, DoRF projects the recovered motion onto uniformly sampled directions on the unit sphere. This process produces a spherical and spatially homogeneous multi-view representation that captures motion dynamics while reducing sensitivity to variations in multipath observation geometry.

Building on this representation, DoRF++ combines DoRF with attention-based spherical Transformers for activity classification. Experiments on a challenging hand-gesture dataset demonstrate that DoRF++ generalizes effectively to unseen users and noticeably outperforms state-of-the-art Wi-Fi-based HAR methods, particularly for difficult gesture pairs and in the practical single-receiver-AP setting. These results demonstrate the potential of the proposed spherical Doppler representation learning framework, paving the way for robust and generalizable real-world Wi-Fi sensing.

\phantomsection
\section*{AI Use Disclosure}
\addcontentsline{toc}{section}{AI Use Disclosure}
The authors used generative AI tools to help edit the language of this manuscript, draft responses to reviewer comments, and clarify its technical descriptions. All resulting text was reviewed, edited, and verified by the authors, who take full responsibility for the final manuscript.

\phantomsection
\section*{Appendix\\
PROOF OF PROPOSITION~\ref{lem:unique_mod_rot}}
\label{Appendix:A}
\addcontentsline{toc}{section}{Appendix: PROOF OF PROPOSITION~\ref{lem:unique_mod_rot}}
Since $\mathbf{V}_r$ has rank 3, both $\mathbf{V}_1$ and $\mathbf{V}_2$ have full column rank. Therefore, there exists an invertible matrix $\mathbf{A}\in\mathbb{R}^{3\times 3}$ such that $\mathbf{V}_2=\mathbf{V}_1\mathbf{A}$ and \rev{$\mathbf{R}_2=\mathbf{A}^{-1}\mathbf{R}_1$}. Since $\mathbf{V}_1$ has full column rank, one valid explicit choice for $\mathbf{A}$ is $\mathbf{A}=(\mathbf{V}_1^\top \mathbf{V}_1)^{-1}\mathbf{V}_1^\top \mathbf{V}_2$. Unit-norm constraints give
\[
\begin{gathered}
1=\|\mathbf{r}_{2,i}\|^2
=
\rev{\|\mathbf{A}^{-1}\mathbf{r}_{1,i}\|^2}
=
\mathbf{r}_{1,i}^\top\mathbf{B}\,\mathbf{r}_{1,i},\\
\rev{\mathbf{B}\triangleq \mathbf{A}^{-\top}\mathbf{A}^{-1}},
\end{gathered}
\]
and thus $\mathrm{tr}(\mathbf{B}\,\mathbf{r}_{1,i}\mathbf{r}_{1,i}^\top)=1$ for all $i$.

By the spanning assumption, these constraints determine $\mathbf{B}$ uniquely. Indeed, if two symmetric matrices $\mathbf{B}$ and $\mathbf{B}'$ both satisfy them, their difference $\boldsymbol{\Delta}\triangleq \mathbf{B}-\mathbf{B}'$ is symmetric and obeys $\mathrm{tr}(\boldsymbol{\Delta}\,\mathbf{r}_{1,i}\mathbf{r}_{1,i}^\top)=0$ for all $i$. Since $\{\mathbf{r}_{1,i}\mathbf{r}_{1,i}^\top\}_{i=1}^N$ spans the space of symmetric $3\times 3$ matrices, every symmetric $\mathbf{S}=\sum_i \chi_i\,\mathbf{r}_{1,i}\mathbf{r}_{1,i}^\top$ satisfies $\mathrm{tr}(\boldsymbol{\Delta}\mathbf{S})=0$ by linearity of the trace. Choosing $\mathbf{S}=\boldsymbol{\Delta}$ gives $\mathrm{tr}(\boldsymbol{\Delta}^2)=\|\boldsymbol{\Delta}\|_F^2=0$, hence $\boldsymbol{\Delta}=\mathbf{0}$ and $\mathbf{B}=\mathbf{B}'$. Since $\mathbf{B}=\mathbf{I}_3$ satisfies the constraints (because $\|\mathbf{r}_{1,i}\|=1$), it is the unique feasible solution. Hence \rev{$\mathbf{A}^{-\top}\mathbf{A}^{-1}=\mathbf{I}_3$}, so \rev{$\mathbf{A}\mathbf{A}^\top=\mathbf{I}_3$} and $\mathbf{A}$ is orthogonal. Setting $\mathbf{Q}=\mathbf{A}$ \rev{gives \(\mathbf{V}_2=\mathbf{V}_1\mathbf{Q}\) and \(\mathbf{R}_2=\mathbf{Q}^\top\mathbf{R}_1\)}, completing the proof.

\bibliographystyle{IEEEtran}

\bibliography{refs}




\begin{IEEEbiography}[{\includegraphics[width=1in,height=1.25in,clip,keepaspectratio]{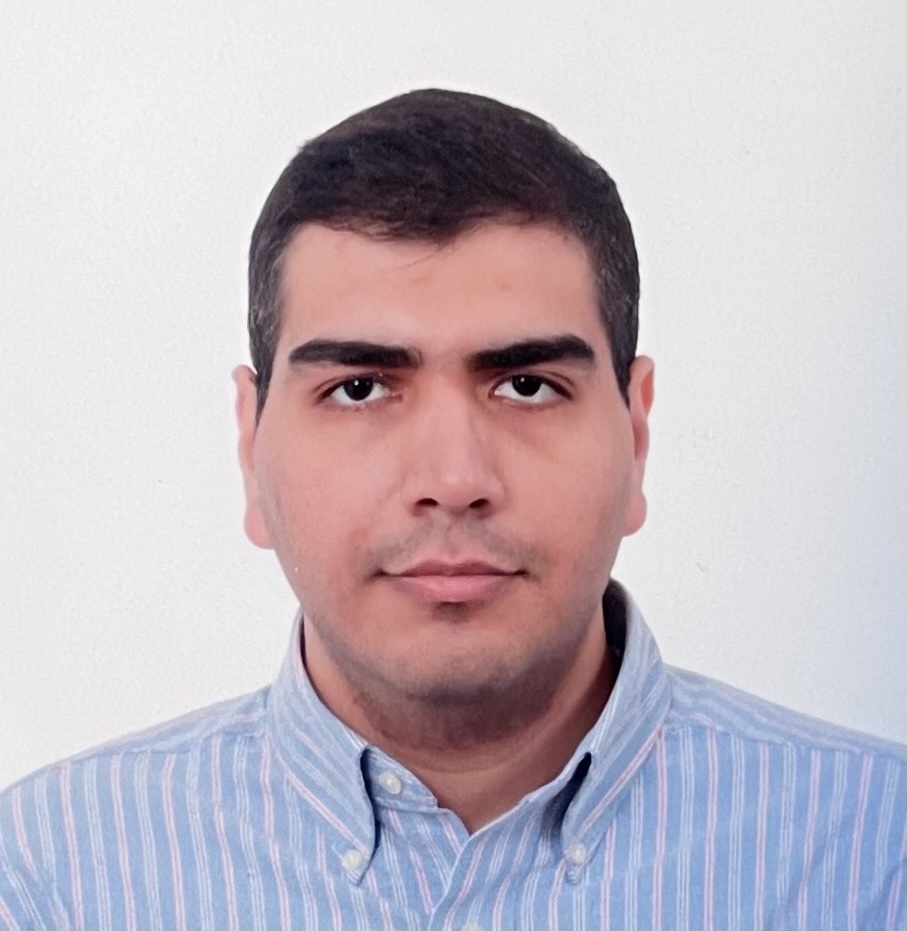}}]{Navid Hasanzadeh}
is currently a Ph.D. candidate in the Edward S. Rogers Sr. Department of Electrical and Computer Engineering at the University of Toronto, Toronto, ON, Canada. His research interests primarily include wireless sensing and machine learning for healthcare. He leverages signal processing techniques and machine learning to drive advancements in Wi-Fi-based human activity recognition, digital healthcare, and wearable health monitoring, with a focus on enhancing the generalization capabilities of machine learning models to ensure reliable performance across diverse scenarios and real-world deployments.
\end{IEEEbiography}


%
\begin{IEEEbiography}[{\includegraphics[width=0.9in,height=1.25in,clip,keepaspectratio]{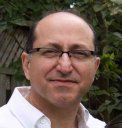}}]{Shahrokh Valaee} \revised{is a Professor with the Edward S. Rogers Sr. Department of Electrical and Computer Engineering, University of Toronto, and the holder of the Nortel Chair of Network Architectures and Services. He is the Founder and the Director of the Wireless Innovation Research Laboratory (WIRLab) at the University of Toronto. 

Professor Valaee was the TPC Co-Chair and the Local Organization Chair of the IEEE Personal Mobile Indoor Radio Communication (PIMRC) Symposium 2011, the TPC Co-Chair of ICT 2015, PIMRC 2017, and the Track Co-Chair of WCNC 2014, PIMRC 2020, and VTC Fall 2020. He was the co-chair of the organizing committee for PIMRC 2023. He is currently a member of the Steering Committee of IEEE PIMRC. From December 2010 to December 2012, he was the Associate Editor of the IEEE Signal Processing Letters. From 2010 to 2015, he served as an Editor of IEEE Transactions on Wireless Communications. Currently, he is an Editor of IEEE Transactions on Wireless Communications and an Associate Editor of the Journal of Computer and System Sciences. From 2021 to 2023, he was a Distinguished Lecturer of the IEEE Communications Society. Currently, he serves as a Distinguished Lecturer for the IEEE Vehicular Technology Society. He was the co-recipient of the best paper award in the IEEE Machine Learning for Signal Processing (MLSP) 2020 workshop. Professor Valaee is a Fellow of the Engineering Institute of Canada and a Fellow of IEEE. }
\end{IEEEbiography}

\vfill

\end{document}